\documentclass{article}

\usepackage{microtype}
\usepackage{graphicx}
\usepackage{booktabs}

\usepackage{hyperref}
\usepackage{multirow}

\usepackage{arxiv}

\usepackage{amsmath}
\usepackage{amssymb}
\usepackage{mathtools}
\usepackage{amsthm}

\theoremstyle{plain}
\newtheorem{theorem}{Theorem}[section]
\newtheorem{proposition}[theorem]{Proposition}
\newtheorem{lemma}[theorem]{Lemma}
\newtheorem{corollary}[theorem]{Corollary}
\theoremstyle{definition}

\theoremstyle{remark}
\newtheorem{remark}[theorem]{Remark}

\newcommand{\R}{\mathbb{R}}
\newcommand{\E}{\mathbb{E}}
\newcommand{\Eh}{\widehat{\mathcal{E}}}

\DeclareMathOperator*{\argmin}{arg\,min}

\newcommand{\dhc}[1]{}

\newcommand{\ntc}[1]{}
\newcommand{\lpc}[1]{}

\newcommand{\mvc}[1]{}

\icmltitlerunning{On the Convergence of Multicalibration Gradient Boosting}

\begin{document}

\icmltitle{On the Convergence of Multicalibration Gradient Boosting}

\begin{icmlauthorlist}
\icmlauthor{Daniel Haimovich}{meta}
\icmlauthor{Fridolin Linder}{meta}
\icmlauthor{Lorenzo Perini}{meta}
\icmlauthor{Niek Tax}{meta}
\icmlauthor{Milan Vojnovi\' c}{lse}
\end{icmlauthorlist}

\icmlaffiliation{meta}{Meta, London, United Kingdom}
\icmlaffiliation{lse}{Department of Statistics, LSE, London, United Kingdom}

\icmlcorrespondingauthor{Daniel Haimovich}{danielha@meta.com}
\icmlcorrespondingauthor{Fridolin Linder}{flinder@meta.com}
\icmlcorrespondingauthor{Lorenzo Perini}{lorenzoperini@meta.com}
\icmlcorrespondingauthor{Niek Tax}{niek@meta.com}
\icmlcorrespondingauthor{Milan Vojnovi\' c}{m.vojnovic@lse.ac.uk}

\icmlkeywords{Multicalibration, Gradient Boosting, Convergence}

\vskip 0.3in

\printAffiliationsAndNotice{}

\begin{abstract}
Multicalibration gradient boosting has recently emerged as a scalable method that empirically produces approximately multicalibrated predictors and has been deployed at web scale. Despite this empirical success, its convergence properties are not well understood. In this paper, we provide computational guarantees for multicalibration gradient boosting algorithms. We show that the magnitude of successive prediction updates decays at $O(1/\sqrt{T})$ (in the minimum over $T$ rounds), which implies the same convergence rate bound for the empirical multicalibration error over rounds. For the squared-error loss, this rate improves to linear convergence under additional smoothness assumptions on the weak learners, and convergence extends to relaxed and adaptive rescaling variants. Combining the empirical bound with uniform convergence arguments, we also bound the population multicalibration error for the log-loss under additional conditions on the weak learners, which yields a prescription for the number of rounds as a function of the sample size. Experiments on real-world datasets support our theory, clarify the regimes in which the method achieves fast convergence, and show that early stopping is needed on held-out data.

\end{abstract}

\section{Introduction}
Multicalibration has emerged as a rigorous criterion for trustworthy machine learning, requiring predictions to match true outcome expectations not only globally, but across a virtually unlimited class of (potentially overlapping) subpopulations \citep{hebert2018,haghtalab2023,pfisterer2021}. While originally applied for algorithmic fairness, multicalibration has proven relevant for general model reliability, improving (1) robustness against distribution shifts \citep{kim2022,wu2024} and (2) performance in downstream tasks ranging from ranking to uncertainty quantification \citep{yang2024,detommaso2024,dwork2022}.

To achieve this in practice, a new class of multicalibration gradient boosting methods has recently been developed \citep{tax2026,jin2025}. Unlike traditional post-processing methods that rely on grid discretization, these algorithms iteratively refine a predictor by training an ensemble of weak learners on residuals, dynamically augmenting the feature space at each round with the predictions from the previous round. Such an approach has shown significant empirical success: it requires no manual group specification, scales to large datasets and, as highlighted by MCGrad \citep{tax2026}, can be deployed in production environments at web scale.

\textbf{Classical versus Multicalibration Gradient Boosting.} Classical gradient boosting fits weak learners on a fixed feature space to the residuals of the current ensemble; it is a greedy functional gradient method on a fixed convex objective over the span of a fixed weak-learner class, with well-understood convergence \citep[e.g.,][]{lu2020}. Multicalibration gradient boosting differs in two ways: it has two loops (each outer round runs a full gradient boosting machine), and the weak learners of round $t+1$ take the current prediction $f_t(x)$ as an additional input, so the directions available to the algorithm change from round to round. The outer loop is thus not gradient descent on a fixed objective over a fixed function class. With an exact oracle and the squared or log-loss, its fixed points are exactly the predictors that are empirically multicalibrated with respect to the weak learners evaluated at their own predictions (Section 3). Gradient-boosted trees suit this task: with the current prediction as an extra feature, tree splits find the score and feature regions where the model is miscalibrated without manual group specification, repeated rounds correct miscalibration that shifts between regions, and fast implementations with logit rescaling, a minimum Hessian per leaf, and early stopping keep training and inference cheap \citep{tax2026}; early stopping is also what our analysis and experiments recommend (Sections 5--6).

Despite this empirical adoption, the theory that supports multicalibration gradient boosting methods remains under-explored. Existing guarantees are either single-round and require a loss-saturation condition \citep{jin2025}, or bound empirical multicalibration error by the prediction gap without proving it vanishes \citep{tax2026}. Because the predictions at the $t$-th round become a feature input for round $t + 1$, the optimization target shifts at every round, making existing convergence analysis for standard gradient boosting inapplicable for this setting. Thus, fundamental computational questions remain unanswered: does the sequence of predictors actually converge? Does the empirical multicalibration error converge to zero, and if so, how fast is this convergence? Can we guarantee stability when rescaling updates? And does the guarantee carry over to the population, and how many rounds should be run? Such computational guarantees establish algorithm correctness and guide hyperparameter selection.

In this paper, we address the aforementioned computational questions, i.e., identifying conditions under which the empirical multicalibration error converges to zero and, furthermore, establishing convergence rate bounds. Our analysis is focused on studying the convergence of the gap between successive prediction values, as it implies convergence for the empirical multicalibration error (Section 3). Our main convergence result accommodates both classification and regression tasks under different loss functions. We prove that for the standard algorithm, the gap between successive prediction values converges to zero with a decay rate of $O(1/\sqrt{T})$ (in the minimum over rounds), where $T$ is the number of rounds, provided that the boosting oracle is sufficiently accurate in returning a nearly-optimal ensemble of weak learners. We provide further results for regression task with squared-error loss function, under the assumption of exact boosting oracle. Under smoothness assumptions on the weak learners, we show a linear convergence rate (i.e., the prediction gap shrinks by a factor $\kappa < 1$ per round when the projector onto the weak learners varies slowly enough). For the squared-error loss, we also extend our analysis to practical variants used to mitigate overfitting. We show that the $O(1/\sqrt{T})$ rate holds even when using relaxed rescaling weights $w_t \in (0,1]$ that converge to 1 with $\sum_t(1 - w_t) < \infty$. We analyze an adaptive variant whose rescaling weight is set in each round by a line search on the training loss, and show that it keeps the loss non-increasing and, for bounded weak learners, drives the empirical multicalibration error to zero at an $O(1/\sqrt{T})$ rate in the minimum over rounds. Beyond the training set, we bound the \emph{population} multicalibration error via uniform convergence over the class of $T$-round predictors (Section 5): for the log-loss and weak learners that are bounded and Lipschitz in the prediction feature, it is of order $(p/n)^{1/6}$ (up to constants and logarithmic factors) when $T$ grows like $(n/p)^{1/3}$. We present empirical results validating our theory on the training set and on held-out data (Section 6).

\section{Related Work}
\textbf{Foundations of Multicalibration.} Multicalibration \citep{hebert2018} is connected to outcome indistinguishability \citep{dwork2021}, omniprediction \citep{gopalan2022a,hu2023}, boosting for regression \citep{globus2023}, and swap-agnostic learning \citep{gopalan2023}; related notions include multiaccuracy \citep{kim2019}, low-degree multicalibration \citep{gopalan2022b}, HappyMap \citep{deng2023}, moment multicalibration \citep{jung2021}, and multivalidity \citep{gupta2022,jung2023} (Appendix E). Our work is different in providing the first finite-time convergence rate guarantees for recently proposed multi-round multicalibration gradient boosting algorithms \citep{tax2026} and considering both classification and regression tasks.

\textbf{Analysis of Gradient Boosting.} Classical gradient boosting is well studied, e.g., computational guarantees for (randomized) gradient boosting machines \citep{lu2020} and convergence and consistency of boosting with early stopping \citep{zhang2005}. These analyses assume fixed features, i.e., a single round of our algorithm; we analyze the outer rounds, in which the prediction is fed back as a feature.

\textbf{Multicalibration via Gradient Boosting.} \citet{jin2025} proposed a tree-ensemble method trained on features augmented with the current prediction and showed that approximate multicalibration can be achieved in a single round under a restrictive loss saturation condition. \citet{tax2026} introduced MCGrad, a recursive GBDT method deployed at web-scale, and bounded the empirical multicalibration error by a function of the gap between successive predictions, without establishing that this gap converges to zero. Existing approaches thus lack convergence guarantees for arbitrarily many rounds (Appendix E). Regarding generalization, \citet{tax2026} give no population guarantee, and the finite-sample analysis of \citet{jin2025} concerns a single round under the loss saturation condition; it is therefore not directly comparable with our population bound for the multi-round algorithm (Theorem~\ref{thm:pop}), which requires no such condition.

\section{Preliminaries and Problem Setup}
\textbf{Multicalibration and Notation.} Let $(x_1, y_1), \dots ,(x_n, y_n)$ be data points taking values in $\R^d \times \R$, assumed to be random i.i.d.\ from a distribution $D$. We denote the vector of targets as $\mathbf{y} = (y_1, \dots, y_n)^\top$ and a predictor as $f : \R^d \to \R$, with $f := (f(x_1), \dots, f(x_n))^\top \in \R^n$ also denoting the prediction vector when the context is unambiguous.

Let $\mathcal{B}$ be a finite set of functions (weak learners), where $b_j : \R^d \times \R \to \R$ for each $b_j \in \mathcal{B}$. The cardinality of $\mathcal{B}$ is denoted by $p$. A key feature of multicalibration boosting methods is that predictors interact with other predictions: intuitively, $b_j(x, f)$ represents a predicted response for a given feature vector $x$ and an auxiliary prediction feature $f$. For any prediction vector $f \in \R^n$ (where $f_i$ corresponds to $x_i$), we define the matrix-valued function $B(f) \in \R^{n\times p}$ such that the $j$-th column is $b_j(f) := (b_j(x_1, f_1), \dots , b_j(x_n, f_n))^\top$.

Note that the set of weak learners can be arbitrarily large but finite. This assumption allows analysis in a finite-dimensional space and covers common weak learners such as regression trees and parametric models with floating-point parameters.

The multicalibration errors of a predictor $f$ with respect to the hypothesis class of functions $\mathcal{B}$ are defined as the vector $\mathcal{E}(f) \in \R^p$ with components $\epsilon_j = \E_{(x,y)\sim D}[(y - \phi(f(x)))b_j(x, f(x))]$ for $j = 1, \dots , p$, where $\phi$ is the link function (identity for squared loss; sigmoid for log-loss). The predictor $f$ is multicalibrated with respect to the hypothesis class $\mathcal{B}$ if $\mathcal{E}(f) = 0$. This can be extended naturally to the relaxed notion of $\alpha$-multicalibration, which requires $\|\mathcal{E}(f)\| \le \alpha$ for some norm function $\|\cdot\|$. The multicalibration error can also be seen as the correlation between the residuals and the weak learners evaluated at the current predictions.

Similarly, the empirical multicalibration errors are defined as the sample means over the given dataset, which can be written as:
\begin{equation}
\Eh(f) := \frac{1}{n} B(f)^\top (\mathbf{y} - \phi(f)).
\end{equation}
We say that a loss $L(y, f)$ has a \emph{canonical gradient} (with respect to the link $\phi$) if $\partial L(y, f)/\partial f = \phi(f) - y$ for all $y$ and $f$. This holds for the squared loss ($\phi = \mathrm{id}$) and for the log-loss written in terms of the logit $f$ ($\phi$ the sigmoid). For such losses, the empirical multicalibration error is the gradient of the training loss $V$ (defined below) projected onto the weak learners: $\Eh(f) = -\frac{1}{n}B(f)^\top\nabla V(f)$.

\textbf{Factorized Hypothesis Classes.} A special type of the set $\mathcal{B}$ consists of functions that admit the following factorization. Let $\mathcal{H} = \{h : \R^d \to \R\}$ and $\mathcal{G} = \{g : \R \to \R\}$ be finite sets of functions with cardinalities $m$ and $k$, respectively. For every $b \in \mathcal{B}$, we have $b(x, u) = h(x)g(u)$ for some $g \in \mathcal{G}$ and $h \in \mathcal{H}$. For this type of function, the matrix $B(f)$ can be expressed as follows. Let $h(x) = (h_1(x), \dots , h_m(x))^\top$ and $g(u) = (g_1(u), \dots , g_k(u))^\top$. Define (using boldface to disambiguate from the scalar $b$ above):
\begin{equation}
\mathbf{b}(x, u) = h(x) \otimes g(u) \in \R^{p}
\end{equation}
where $\otimes$ denotes the Kronecker product\footnote{For any two vectors $u \in \R^m$ and $v \in \R^k$, $u \otimes v = (u_1 v^\top, \dots , u_m v^\top)^\top \in \R^{mk}$.} and $p = mk$, and let
\begin{equation}
B(f) = (\mathbf{b}(x_1, f_1), \dots , \mathbf{b}(x_n, f_n))^\top.
\end{equation}

\textbf{The Multicalibration Boosting Framework.} We analyze the class of multicalibration gradient boosting algorithms \citep{tax2026,jin2025}. Unlike standard boosting, these methods dynamically update the feature space at every round by including the model's current prediction. For a loss $L : \R^2 \to \R$, let $V(f) := \sum_{i=1}^n L(y_i, f(x_i))$ denote the training loss. The procedure is formalized in Algorithm~\ref{alg:mcgb}: at round $t$, the algorithm fits an ensemble of weak learners $h_t = B(f_t)\theta_t$, whose coefficients $\theta_t \in \R^p$ (approximately) minimize $\theta \mapsto V(f_t + B(f_t)\theta)$, with features augmented by $f_t$ (equivalent to fitting residuals $y - f_t$ for squared loss).

\begin{algorithm}[h]
\caption{Multicalibration Gradient Boosting}\label{alg:mcgb}
\begin{algorithmic}[1]
\Require Initial predictor $f_0$, data $(X, \mathbf{y})$, step size $\eta \in (0, 1]$, number of rounds $T > 0$.
\For{$t = 0, \dots , T - 1$}
  \State 1. Compute the matrix $B(f_t)$ using the current predictions $f_t$.
  \State 2. Fit the ensemble $h_t = B(f_t)\theta_t$, with coefficients $\theta_t \in \R^p$ that (approximately) minimize $\theta \mapsto V(f_t + B(f_t)\theta)$.
  \State 3. Update: $f_{t+1} = w_t(f_t + \eta h_t)$.
\EndFor
\end{algorithmic}
\end{algorithm}

The candidate predictor at round $t$ is rescaled by a weight $w_t$ to obtain the predictor for round $t + 1$. By default $w_t = 1$ (no rescaling). We also consider relaxed rescaling weights following a fixed schedule and adaptive rescaling weights depending on prediction values, which we analyze for the squared-error loss; see Section~4.2.2 for details.

Abstracting from the inner boosting iterations, the algorithm induces the following discrete-time dynamical system:
\begin{equation}
f_{t+1} = w_t(f_t + \eta B(f_t)\theta_t)
\end{equation}
where $\theta_t \in \R^p$ are the coefficients assigned to weak learners in round $t$ by the (possibly inexact) boosting oracle.

For the squared-error loss function and an exact oracle, the dynamical system in (4) can be written as follows:
\begin{equation}
f_{t+1} = w_t(f_t + \eta A(f_t)(\mathbf{y} - f_t)),
\end{equation}
where $A(f) := B(f)B(f)^+$ is the orthogonal projector onto the column space of $B(f)$, and $B(f)^+$ denotes the Moore-Penrose inverse of $B(f)$. The term $A(f_t)(\mathbf{y}-f_t)$ represents the best approximation of the residuals using the hypothesis class $\mathcal{B}$ conditioned on the current predictions.

\textbf{Why the Algorithm Works, and the Prediction Gap.} Each round (approximately) minimizes the loss over combinations of weak learners evaluated at the current predictions. For losses with a canonical gradient, such as the squared and log-losses, the first-order optimality condition of this problem is
\[
B(f_t)^\top\big(\mathbf{y} - \phi(f_t + B(f_t)\theta_t)\big) = 0,
\]
i.e., the candidate $f_t + B(f_t)\theta_t$ has zero empirical multicalibration error with respect to the weak learners evaluated at the \emph{previous} predictions $f_t$ \citep[a similar argument appears in][]{tax2026}; this is exact for an exact oracle and approximate otherwise. As the candidate differs from $f_t$ by $\eta^{-1}(f_{t+1} - f_t)$, it remains to show that this gap shrinks, since $\|\Eh(f_t)\|$ is bounded by a multiple of $\|f_{t+1} - f_t\|$ plus an oracle noise floor (Eq.~(6)). The gap evolves according to (4)--(5), analyzed in Section 4. With an exact oracle and $w_t = 1$, $f_t$ is a fixed point ($h_t = 0$) exactly when $B(f_t)^\top\nabla V(f_t) = 0$, i.e., for losses with a canonical gradient, exactly when $\Eh(f_t) = 0$. Every rate on the gap thus translates into a rate for $\Eh$.

\section{Main Results}
We show convergence of (4) for general smooth losses and, for the squared-error loss, linear rates and extensions to rescaling. Full proofs are in Appendix A.

\textbf{Proof Techniques and What Is New.} Standard arguments are flagged after each result; the parts specific to multicalibration boosting are: (i) the reduction of the empirical multicalibration error to the prediction gap via first-order optimality of each round's subproblem (Section 3), which turns a multicalibration question into the convergence of a dynamical system whose feature matrix $B(f_t)$ changes every round; (ii) pseudoinverse perturbation theory \citep{stewart1977} to control how fast the projector onto the weak-learner span moves, giving checkable conditions for linear convergence (Lemmas 4.3--4.4); (iii) an explicit noise floor that separates whether the outer dynamical system converges from whether the inner boosting oracle is accurate (Theorem 4.1, Eq.~(6)); and (iv) a discussion, with an empirical check, of where the Lipschitz assumptions strain against hard-split trees (Sections 4.2.1 and 6). For the population bound, the new ingredient is a covering argument through the $T$ rounds for this self-referential structure (Section 5).

We introduce the following assumptions on the training loss $V$ of Section 3:

\begin{itemize}
\item[(A1)] \textbf{Smoothness.} $V(f)$ is $L$-smooth on the segment $\{\tau f_t + (1 - \tau)f_{t+1} : 0 \le \tau \le 1\}$ for all $t = 0, \dots , T - 1$.
\item[(A2)] \textbf{Strong convexity.} $V(f)$ is $m$-strongly convex on $\{f_t + \tau h_t : 0 \le \tau \le 1\}$ for all $t = 0, \dots , T - 1$, where $h_t = B(f_t)\theta_t$.
\item[(A3)] \textbf{Inexact boosting oracle.} At each step $t$, for some $\epsilon_t \ge 0$, the weights $\theta_t$ satisfy
\[
V(f_t + B(f_t)\theta_t) \le \min_{\theta \in \R^p} V(f_t + B(f_t)\theta) + \epsilon_t.
\]
\end{itemize}

For the case of regression with squared-error loss function, i.e., when $L(y, f) = \frac{1}{2}(y - f)^2$, smoothness and strong convexity hold globally for every $f$ with $L = 1$ and $m = 1$.

For the case of binary classification with log-loss, $V(f)$ is globally smooth with $L = 1/4$, but its strong convexity vanishes as $\|f\|_\infty \to \infty$, so (A2) requires $\max_{0\le\tau\le1}\|f_t + \tau h_t\|_\infty$ to be bounded for all $t$ (details in Appendix A.13); this holds for any fixed $T$ if no round's weak learners (quasi-)separate the labels (Lemma~\ref{lem:logloss-exist}).

The boosting oracle is said to be exact when condition (A3) holds with $\epsilon_t = 0$. In this case, in each round $t$, the boosting machine outputs a linear combination of weak learners $B(f_t)\theta_t$, where $B(f) \in \R^{n\times p}$ is the matrix of weak learner evaluations defined in Eq.~(3), and $\theta_t$ is the optimal coefficient vector $\theta_t \in \R^p$ that minimizes the loss function, i.e.,
\[
\theta_t := \argmin_{\theta \in \R^p} \Big\{ L_t(\theta) := \sum_{i=1}^n L\Big(y_i, f_t(x_i) + \sum_{j=1}^p b_j(x_i, f_t(x_i))\theta_j\Big) \Big\}.
\]
The exact boosting oracle assumption is justified by fast convergence of gradient boosting machines; e.g., in \citet{lu2020}, it was shown that, under suitable conditions, gradient boosted machines converge exponentially fast for any smooth, strongly-convex loss function. In this case, the approximation error parameter $\epsilon_t$ can be controlled by the number of boosting iterations in round $t$. Within a round the features $(x, f_t(x))$ are fixed, so classical guarantees for the inner gradient boosting machine apply: $\epsilon_t$ decays geometrically (strongly convex losses) or as $O(1/M_t)$ (smooth losses) in the number $M_t$ of inner iterations \citep{lu2020}; Appendix B discusses the conditions and what they imply for the assumptions of Theorem~\ref{thm:main}.

\subsection{Computational Guarantees}
Our key result shows conditions for multicalibration boosting algorithms to have approximately non-increasing training loss, diminishing gap between successive predictions, and convergence of empirical multicalibration error to zero.

\begin{theorem}[Convergence with No Rescaling]\label{thm:main}
Consider the dynamical system in Eq.~(4) with constant unit rescaling weights ($w_t = 1$) and the step size $\eta \le m/(2L)$. Then:
\begin{enumerate}
\item \textbf{Approximately Non-increasing Training Loss:} under conditions (A2)--(A3),
\[
V(f_{t+1}) \le V(f_t) + \epsilon_t.
\]
\item \textbf{Diminishing Gap between Successive Predictions:} Under conditions (A1)-(A3), (a) if $\sum_{t=0}^\infty \epsilon_t < \infty$, then the gap between successive predictions satisfies $\|f_{t+1} - f_t\|_2 \to 0$, and (b) the gaps between successive predictions are bounded as
\[
\frac{1}{T}\sum_{t=0}^{T-1} \|f_{t+1} - f_t\|^2 \le \frac{4\eta}{m}\left(\frac{V(f_0)}{T} + \frac{\eta}{T}\sum_{t=0}^{T-1}\epsilon_t\right).
\]
Consequently,
\[
\min_{0\le t\le T-1} \|f_{t+1} - f_t\| \le \sqrt{\frac{4\eta V(f_0)}{m}}\,\frac{1}{\sqrt{T}} + \sqrt{\frac{4\eta^2}{m}\,\frac{1}{T}\sum_{t=0}^{T-1}\epsilon_t}.
\]
\end{enumerate}
\end{theorem}

For the case of the exact boosting oracle, the loss is non-increasing, and the second term in the bound on the minimum gap between successive predictions is equal to zero; the step-size condition then relaxes to $\eta \le m/L$, and the constant 4 improves to 2 (Appendix A.1). When $\epsilon_t = \epsilon$ is constant, the bound converges to the noise floor $2\eta\sqrt{\epsilon/m}$. When $\frac{1}{T}\sum_{t=0}^{T-1}\epsilon_t \to 0$ as $T \to \infty$, the bound still vanishes. Specifically, when $\sum_{t=0}^{T-1}\epsilon_t = O(1)$, the bound is still $O(1/\sqrt{T})$.

\emph{Proof idea (standard first-order argument).} Strong convexity (A2) with near-optimality (A3) upper-bounds $\nabla V(f_t)^\top h_t$, and smoothness (A1) upper-bounds the loss at the shrunk step in terms of it; the unknown gradient term cancels and telescoping bounds $\sum_t\|f_{t+1} - f_t\|^2$ (Appendix A.1). This is the usual descent-lemma sandwich; what is specific here is that only per-round properties of the subproblem are used, although $B(f_t)$ changes at every round.

The convergence rate bounds on the gap between successive predictions imply convergence rate bounds on the empirical multicalibration error. For a loss with a canonical gradient (Section 3), i.e., $\nabla V(f) = \phi(f) - \mathbf{y}$, and $V$ globally $L$-smooth (e.g., the squared loss, $\phi = \mathrm{id}$ and $L = 1$, and the log-loss, $L = 1/4$), with an inexact oracle and unit-valued rescaling weights, we have (Lemma~\ref{lem:A1}):
\begin{equation}
\|\Eh(f_t)\|_2 \le \frac{1}{n}\|B(f_t)\|_2\left(\frac{L}{\eta}\|f_{t+1} - f_t\|_2 + \sqrt{2L\epsilon_t}\right).
\end{equation}
The term $\sqrt{2L\epsilon_t}$ constitutes a noise floor: even after the prediction gap vanishes, the multicalibration error can be no better than $O\big(\|B(f_t)\|_2\sqrt{L\epsilon_t}/n\big)$. For such losses, a group-normalized multicalibration error is also bounded by the gap (Appendix A.3).

This theorem provides the first theoretical justification for the success of boosting algorithms like MCGrad \citep{tax2026}. It guarantees that, despite the changing feature space, the algorithm converges to a limit set of points and effectively minimizes the empirical multicalibration error. In subsequent sections we provide further convergence results for the case of exact boosting oracle and regression tasks with the squared-error loss function.

\subsection{Additional Results for Regression with Squared-Error Loss and Exact Boosting Oracles}
\subsubsection{When is Convergence Fast?}
While Theorem~\ref{thm:main} guarantees sublinear convergence, empirical insights (see Section~6) often suggest faster (linear) convergence. Therefore, we investigate the structural properties of the hypothesis class $\mathcal{B}$ required to allow such faster convergence rates.

\begin{theorem}[Linear Convergence under Smoothness]\label{thm:linear}
Assume $A(f)$ is Lipschitz continuous with constant $L_A$ along the trajectory $\{f_t\}$ and let $\kappa := 1 - \eta + \eta L_A\|\mathbf{y} - f_0\|_2$. Then, for every $t \ge 1$:
\[
\|f_{t+1} - f_t\|_2 \le \kappa\|f_t - f_{t-1}\|_2.
\]
As a result, if $\kappa < 1$, then the gap converges to zero linearly (geometrically fast).
\end{theorem}

Theorem~\ref{thm:linear} applies to weak learners for which $A(f)$ is Lipschitz continuous. For standard decision-tree-based weak learners, $A(f)$ is in general not Lipschitz continuous due to discontinuities introduced by hard splits. Natural examples of Lipschitz continuous weak learners include soft trees (replacing hard splits with sigmoid gating functions), linear models, and smooth kernel-based weak learners. Despite this lack of strict Lipschitz continuity, in Section~6 we observe geometric decay of the prediction gap when training with hard-split GBDTs. We interpret this as evidence that large ensembles approximately recover the smoothness regime required by Theorem~\ref{thm:linear} (Appendix A.4).

\textbf{Is Lipschitz Continuity Essential?} Only the linear rate of Theorem~\ref{thm:linear} uses the Lipschitz condition on $A(f)$; Theorems~\ref{thm:main}, \ref{thm:relaxed}, and~\ref{thm:adaptive} do not. The proof uses it only between consecutive iterates, $\|A(f_t) - A(f_{t-1})\|_2 \le L_A\|f_t - f_{t-1}\|_2$, a weaker trajectory-level condition that hard-split trees still need not satisfy.

The Lipschitz continuity of the projector $A(f)$ is a non-trivial condition. It holds for smooth weak learners under some spectral properties of matrix $B(f)$ as shown next.

\begin{lemma}[Smoothness of $A(f)$]\label{lem:LA}
Assume that for some $\mathcal{F} \subseteq \R^n$, for every $f \in \mathcal{F}$, $B(f)$ is Lipschitz continuous with constant $L_B$, $B(f)$ has full column rank with smallest singular value at least $\delta > 0$, and $\|B(f)\|_2 \le M$, for some $M > 0$. Then, $A(f)$ is Lipschitz continuous on $\mathcal{F}$ with constant:
\[
L_A \le \frac{2}{\delta}\left(1 + c\,\frac{M^2}{\delta^2}\right)L_B \quad \text{where } c = \sqrt{2}.
\]
\end{lemma}

The condition on the smallest singular value requires the weak learners to be sufficiently diverse on the training data, and the bound $M$ on the largest requires their combined output to stay bounded; an intuitive interpretation is in Appendix A.5.

For the factorized functions common in multicalibration (defined in Section 3), we can explicitly bound the Lipschitz constant $L_B$.

\begin{lemma}[Smoothness of $B(f)$ for Factorized Learners]\label{lem:LB}
Let $\mathcal{B}$ consist of factorized functions $b(x, u) = h(x)g(u)$. If every $g \in \mathcal{G}$ is $L_G$-Lipschitz, then $B(f)$ is Lipschitz (so, under the conditions of Lemma~\ref{lem:LA}, is $A(f)$). Specifically, the Lipschitz constant $L_B$ satisfies:
\[
L_B \le \big(\max_i \|h(x_i)\|_2\big)\sqrt{k}\,L_G.
\]
\end{lemma}

This lemma links the abstract stability of the algorithm to the concrete smoothness of the weak learners ($g(u)$) and the diversity of the features ($H(X)$). This aligns with the intuition that ``sharp'' variations in weak learner behavior (e.g., unstable decision tree splits) might harm fast convergence.

\emph{Proof ideas.} Theorem~\ref{thm:linear} is a direct calculation (Appendix A.4). Lemma~\ref{lem:LA} is the more technical step: it combines a classical perturbation bound for Moore--Penrose pseudoinverses \citep{stewart1977} with careful norm bookkeeping (Appendix A.5); Lemma~\ref{lem:LB} uses the Kronecker structure (Appendix A.6).

\subsubsection{Relaxed and Adaptive Rescaling Strategies}
In practice, the simple additive updates ($w_t = 1$) can lead to overfitting. Practitioners often employ some rescaling of $w_t$ to regularize the predictor. Here, we show that convergence guarantees are robust to such modifications.

\textbf{Relaxed Rescaling.} We first consider the case where the weights $w_t$ approach 1 asymptotically.

\begin{theorem}[Convergence with Relaxed Weights]\label{thm:relaxed}
Assume that rescaling weights $w_t \in (0,1]$ satisfy $\sum_t(1 - w_t) < \infty$. Then, (1) the training loss converges to a limit point, (2) the asymptotic multicalibration (whenever $\|B(f_t)\|_2$ is bounded along the trajectory) and (3) the $O(1/\sqrt{T})$ convergence rate of the gap $\|f_{t+1} - f_t\|_2$ established in Theorem~\ref{thm:main} still hold. Specifically, the following holds:
\[
\min_{0\le t\le T-1}\|f_{t+1} - f_t\|_2 \le \frac{\sqrt{\eta}\|\mathbf{y} - f_0\|_2 + \gamma}{\sqrt{T}}
\]
where $\gamma$ is a constant depending on $\eta$, $\mathbf{y}$, $f_0$ and $\sum_t(1 - w_t)$.
\end{theorem}

Relaxing the weights thus does not break the convergence results, provided $w_t \in (0,1]$ and $\sum_t(1 - w_t) < \infty$. The non-increasing loss property of Theorem~\ref{thm:main} is not guaranteed to hold in the relaxed case; a weaker property that does hold, and the behavior of the constant $\gamma$, are discussed in Appendix A.7.

\textbf{Adaptive Rescaling.} A more aggressive strategy is to optimize $w_t$ at each step. We define the adaptive weight $\omega(f_t)$ as the minimizer of $\|\mathbf{y} - \omega\varphi(f_t)\|_2$, where $\varphi(f_t) = f_t + \eta A(f_t)(\mathbf{y} - f_t)$ is the unscaled update direction. This yields the closed form
\begin{equation}
\omega(f_t) = \frac{\mathbf{y}^\top\varphi(f_t)}{\|\varphi(f_t)\|_2^2}.
\end{equation}
This transforms the system into $f_{t+1} = \omega(f_t)\varphi(f_t)$, effectively projecting $\mathbf{y}$ onto the 1D subspace of the current update.

The following result shows that monotonicity and asymptotic multicalibration are preserved when using adaptive rescaling weights instead of unit-valued rescaling weights.

\begin{theorem}[Convergence with Adaptive Rescaling]\label{thm:adaptive}
Assume that the rescaling weights are adaptive according to Eq.~(7) (with $f_{t+1} = 0$ if $\varphi(f_t) = 0$). Then, (1) the non-increasing property of the loss and (2) the asymptotic multicalibration established in Theorem~\ref{thm:main} still hold; more precisely, $\min_{0\le t\le T-1}\|A(f_t)(\mathbf{y} - f_t)\|_2 \le \|\mathbf{y} - f_0\|_2/\sqrt{\eta(2-\eta)T}$ and, if $M_B := \sup_t\|B(f_t)\|_2 < \infty$, then $\Eh(f_t) \to 0$ and $\min_{0\le t\le T-1}\|\Eh(f_t)\|_2 \le (M_B/n)\|\mathbf{y} - f_0\|_2/\sqrt{\eta(2-\eta)T}$. In addition, (3) for $t \ge 1$ the rescaling weights are such that $\omega(f_t) \ge 0$, i.e., the angle between $\mathbf{y}$ and $\varphi(f_t)$ is acute or right.
\end{theorem}

Appendix A.9 gives a local quadratic rate under an additional condition on the weak learners.

\emph{Proof ideas (more intricate).} Theorem~\ref{thm:relaxed} uses a Lyapunov function with correction terms for the rescaling drift and a Cauchy-sequence argument (Appendix A.7); for Theorem~\ref{thm:adaptive}, comparing the adaptive step with the unscaled one gives a per-round decrease $\eta(2-\eta)\|A(f_t)(\mathbf{y} - f_t)\|_2^2$ that telescopes (Appendix A.8).

\section{Population Multicalibration Error and Generalization}\label{sec:gen}
We now bound the \emph{population} error $\mathcal{E}(f_t)$ of Section 3 via uniform convergence and derive how $T$ should scale with the sample size (proofs in Appendix C). We consider binary classification with log-loss ($\phi$ the sigmoid, $y \in \{0,1\}$); see Remark~\ref{rem:gen-extensions} for other losses. Even for finite $\mathcal{B}$, a union bound over $\mathcal{B}$ does not suffice: $b_j(x, f_t(x))$ depends on $f_t$, which ranges over a data-dependent class built by $T$ rounds, whose complexity must be controlled.

\textbf{Assumptions for the Population Bound.} (G1) The initial predictor $f_0$ is fixed in advance, independent of the sample. (G2) Weak learners are bounded and Lipschitz in the prediction feature: $b(x,u) = (b_1(x,u),\dots,b_p(x,u))$ satisfies $\|b(x,u)\|_2 \le B_2$ and $\|b(x,u) - b(x,u')\|_2 \le L_b|u - u'|$. (G3) The coefficients are bounded: $\|\theta_t\|_2 \le B_\theta$ for all $t$. The constants $B_2, L_b, B_\theta$ are fixed in advance (not chosen from the sample).

Write $\bar V(f_0) := V(f_0)/n$, $\bar M := \max_{t<T}\|B(f_t)\|_2/\sqrt{n}$ (note $\bar M \le B_2$), $\Gamma := 1 + \eta B_\theta L_b$, $K := 2\eta B_\theta(B_2/4 + L_b)$ (if the weak learners take a transformed prediction $\phi(f(x))$ as input, replace $L_b$ by $L_bL_\phi$ in $\Gamma$ and $K$), and
\[
\Lambda := \log\frac{10}{\delta} + p\log(1 + 2n) + pT\log(1 + KnT) + \frac{p}{2}T(T-1)\log\Gamma .
\]

\begin{theorem}[Population multicalibration error]\label{thm:pop}
Consider log-loss with an exact oracle ($\epsilon_t = 0$), $w_t = 1$, $\eta \le m/(2L)$, and (A1)--(A2) with $L = 1/4$. Under (G1)--(G2), for $n \ge 16$ and $\delta \in (0,1)$, with probability at least $1-\delta$ over the training sample, if (G3) holds then the round $t^* \in \arg\min_{0\le t \le T-1}\|f_{t+1} - f_t\|_2$ satisfies
\[
\|\mathcal{E}(f_{t^*})\|_2 \;\le\; \underbrace{2L\bar M\sqrt{\frac{\bar V(f_0)}{\eta m}}\,\frac{1}{\sqrt T}}_{\text{optimization}} \;+\; \underbrace{\bar M\sqrt{\frac{18\Lambda}{n-1}} + \frac{30B_2\Lambda}{n-1}}_{\text{generalization}} .
\]
\end{theorem}

\emph{Proof outline.} (1) \emph{Duality.} $\|\mathcal{E}(f) - \Eh(f)\|_2$ is a supremum over unit $\theta$ of population-minus-empirical means of $e_{\theta,f} = \theta^\top b(x,f(x))(y - \phi(f(x)))$. (2) \emph{Cover of covers.} Their covering number is at most that of the $T$-round predictors times that of the unit-ball weak-learner class (Lemma~C.3). (3) \emph{Covering across rounds.} One boosting round is $\Gamma$-Lipschitz in sup norm, so a cover of the $T$-round predictors is built from one shrunk weak-learner cover per round; its logarithm is at most of order $\tfrac p2T^2\log\Gamma$ (Lemma~C.2). (4) \emph{Uniform convergence.} The Maurer--Pontil bound holds simultaneously over the class, hence at the data-dependent pair $(\theta^*, f_{t^*})$ (Lemma~C.5). (5) \emph{Optimization term.} Lemma~\ref{lem:A1} with Theorem~\ref{thm:main} bounds $\|\Eh(f_{t^*})\|_2$ (Appendix C.3).

\begin{corollary}[Choice of $T$]\label{cor:T}
Let $\Gamma > 1$ be a constant and $p\max\{1, \log\Gamma\} \le n$. Taking $T \asymp (n/(p\log\Gamma))^{1/3}$ gives, with probability at least $1-\delta$,
\[
\|\mathcal{E}(f_{t^*})\|_2 = \tilde O\big((p\log\Gamma/n)^{1/6}\big),
\]
where $\tilde O$ hides the constants in Theorem~\ref{thm:pop} and logarithms of $n, 1/\delta$.
\end{corollary}

The bound is a bias--variance trade-off: $T$ should grow only like $(n/p)^{1/3}$ (compare Section 6). We do not claim $1/6$ is tight\footnote{The exponent is set by the variance term: a weak learner with non-vanishing mean square has a column of norm $\asymp\sqrt n$, so $\|B(f_t)\|_2 \asymp \sqrt n$ and $\bar M = O(1)$. Treating $\|B(f_t)\|_2$ as bounded in $n$ would suggest a faster rate, but with $B_2$ fixed this holds only for weak learners supported on sets of probability $O(1/n)$.}.

\begin{remark}[Scope]\label{rem:gen-extensions}
The result is for unit weights (relaxed and adaptive rescaling are not covered) and does not apply directly to our experiments, which use squared loss and an $f_0$ fit on the same data; see Appendix C.5 for extensions.
\end{remark}

\section{Numerical Results}

\subsection{Experimental Setup}
\textbf{Data and Hyperparameters.} We evaluate the algorithms on five standard regression datasets summarized in Table~\ref{tab:data}; see Appendix D.1 for more details. For all experiments, we use a Random Forest regressor (100 trees, max depth 5) as the initial predictor $f_0$. The weak learners used in the boosting steps are 100 regression trees (max depth 3), implemented via scikit-learn's gradient boosting (learning rate 0.1).

\textbf{MCE Computation.} The multicalibration error is estimated as a sample mean of $\|\Eh(f)\|_2$ over $N_{\mathrm{tree}} = 100$ random regression trees of depth $h = 3$, a class broader than the one used in training (details in Appendix D.2).

\textbf{Rescaling.} We consider three rescaling strategies: \textsc{Unit} sets $w_t = 1$, \textsc{Relaxed} uses the power-law decay $w_t = 1 - (t + 2)^{-3}$, and \textsc{Adaptive} optimizes the weights via Eq.~(7) in each round. We run the algorithms for $T = 20$ rounds with shrinkage $\eta = 0.5$, which satisfies Theorem~4.1's step-size condition $\eta \le m/(2L) = 1/2$ for the squared loss.

\subsection{Experimental Results}

\begin{figure}[!htbp]
\centering
\includegraphics[width=0.99\linewidth]{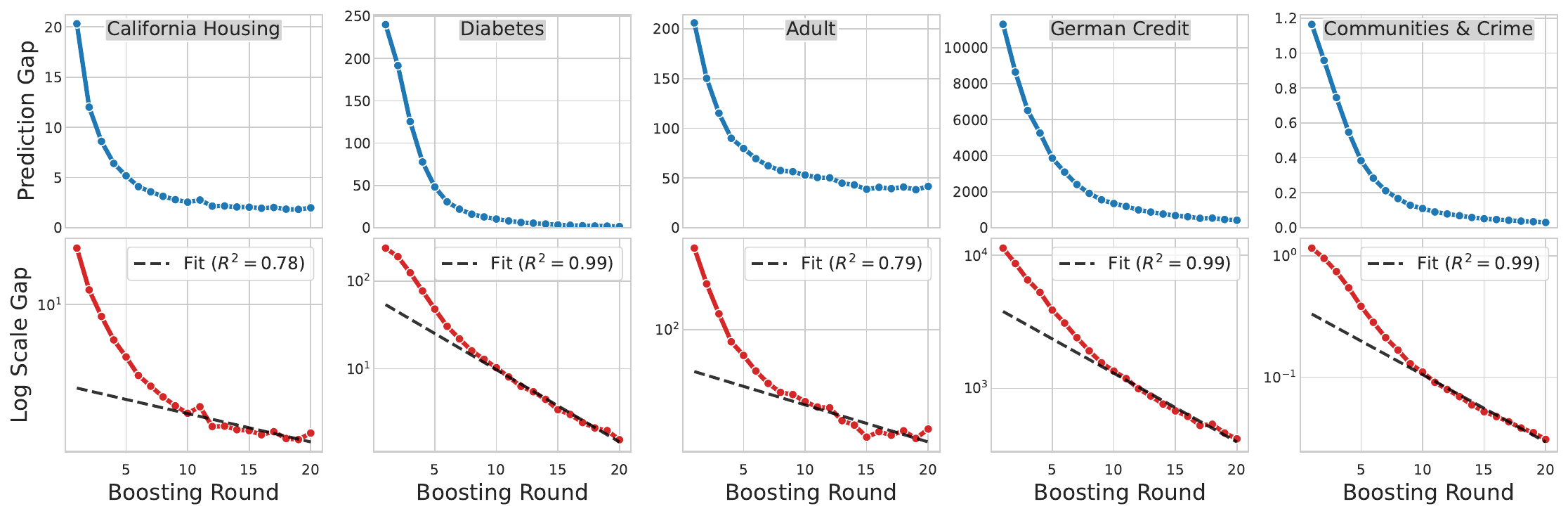}
\caption{The prediction gap $\|f_{t+1} - f_t\|_2$ for all multicalibration boosting rounds $t \le 20$. The bottom row shows the corresponding plots on a log--lin scale, along with the best fitting line (black) and its $R^2$ score.}
\label{fig:gap}
\end{figure}

\textbf{Convergence with Unit Weights.} Figure~\ref{fig:gap} shows a clear overall decay of the prediction gap toward zero on all datasets, with small per-round oscillations ($\le 10\%$) consistent with Theorem~\ref{thm:main} bounding the minimum and average gap. On a semi-log scale, a linear fit to the tail gives $R^2 \in [0.78, 0.99]$; the lower values (California Housing, Adult) reflect a sublinear plateau, consistent with hard-split GBDTs only approximately satisfying the condition of Theorem~\ref{thm:linear} (discussion in Appendix D.3).

\textbf{Convergence with Rescaling Strategies.} On the training set, under all three strategies the training MSE decreases monotonically and the empirical MCE decays smoothly towards 0, often by orders of magnitude (Figure~\ref{fig:mse} and discussion in Appendix D.4).

\textbf{Generalization and Overfitting.} We also examine held-out data, averaging over 20 seeds with a fresh random 80/20 split each (protocol in Appendix D.5). Figure~\ref{fig:test} plots the average test MCE with $\pm 1$ standard-deviation bands.

\begin{figure}[!htbp]
\centering
\includegraphics[width=0.99\linewidth]{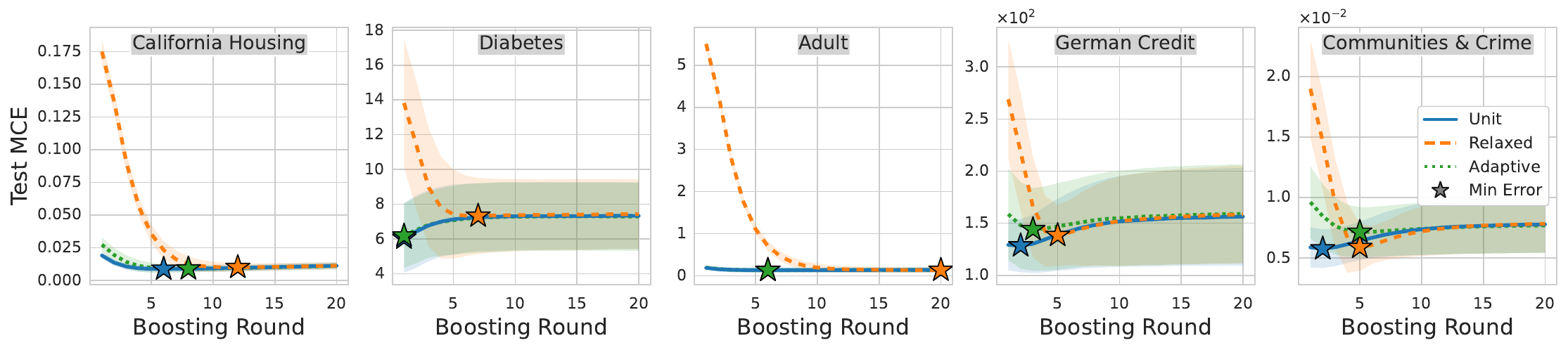}
\caption{Average test Multicalibration Error (MCE) per dataset, with star markers indicating the optimal stopping point for each strategy.}
\label{fig:test}
\end{figure}

\emph{Training versus test.} Unlike the monotonic trend observed on the training set in Figure~\ref{fig:mse} (Appendix D.4), the test error shows the classic bias-variance trade-off. In datasets such as \emph{California Housing} and \emph{German Credit}, we observe a distinct U-shaped evolution, where the error decreases initially before facing overfitting. This confirms that while the dynamical system is stable, regularization via early stopping is necessary in practice. After its minimum, the test MCE rises (by up to about a third on the smaller datasets) and then plateaus, qualitatively as in Theorem~\ref{thm:pop}.

\emph{Dependence on the sample size.} Under Unit weights, the three datasets with $n \le 2{,}000$ reach their minimum test MCE within two rounds, while the two with $n > 20{,}000$ do so around round 6; the other strategies show the same ordering. This is in line with Corollary~\ref{cor:T} (the usable number of rounds grows with $n$), though with five datasets it is only suggestive.

\emph{Role of the weights.} \textbf{Unit} and \textbf{Adaptive} typically reach their optimum within the first 1-8 rounds but can overfit almost immediately (most evident on \emph{German Credit}), whereas \textbf{Relaxed} acts as an implicit regularizer that shifts the optimal stopping point later (around round 12 on \emph{California Housing}, 5 on \emph{German Credit}, and beyond $T = 20$ on \emph{Adult}; Appendix D.5). However, Relaxed's minimum test MCE is not lower than Unit's (it is higher on \emph{Diabetes} and \emph{German Credit}), and the bands overlap on the smaller datasets: its benefit is robustness to the stopping time, not better generalization. A theory of this effect is left to future work; in practice, $T$ should be chosen by validation.

\section{Conclusion}
In this work, we have bridged the gap between the empirical success of multicalibration gradient boosting and its theoretical foundations. While algorithms like MCGrad have been deployed at scale, their dynamic nature -- where predictions recursively define the feature space -- previously lacked rigorous analysis. By modeling this process as a deterministic dynamical system, we provided the first computational guarantees for this class of algorithms.

Specifically, we proved an $O(1/\sqrt{T})$ rate for the prediction gap and hence for the empirical multicalibration error (up to an $O(\sqrt{\epsilon})$ noise floor) for classification and regression; for the squared loss, linear convergence under smoothness and diversity of the weak learners, and $O(1/\sqrt{T})$ rates under relaxed and adaptive rescaling; and, for the log-loss, a population multicalibration bound of order $\tilde O((p\log\Gamma/n)^{1/6})$ (Theorem~\ref{thm:pop}). Our experiments validate the training-set results and show that, on held-out data, the number of rounds must be controlled, in line with the trade-off in Theorem~\ref{thm:pop}.

\textbf{Limitations and Future Work.} Key limitations include: (A2) fails for the log-loss on weak-learner-separable training sets; the linear rate requires a Lipschitz $A(f)$ with well-conditioned $B(f)$, only approximately satisfied by hard-split GBDT ensembles; and the population bound (Theorem~\ref{thm:pop}) covers the log-loss with an exact oracle and unit weights under (G1)--(G3), and so does not apply directly to our experiments. Natural directions include tightness of the bounds, faster rates under broader conditions, infinite weak-learner classes, and the extensions in Appendix C.5.

\section*{Acknowledgements}
Milan Vojnovi\'c\ performed this work in part while a visiting scientist with Meta. We thank Nastaran Okati for early discussions about multicalibration in the regression setting.

\appendix

\section{Proofs and Additional Results for Section 4}
This appendix contains the proofs of the results of Section 4 (Appendices A.1--A.8), including the two bounds on the empirical multicalibration error in terms of the prediction gap (Appendices A.2--A.3). It also contains three additional results: a sufficient condition for quadratic convergence under adaptive rescaling (Appendix A.9), a hybrid boosting algorithm with a preconditioned update (Appendices A.10--A.12), and the existence of the exact oracle for the log-loss (Appendix A.13).

\subsection{Proof of Theorem 4.1}
\emph{Overview.} The proof bounds the change of the training loss in each round from two sides. Strong convexity (A2), together with the near-optimality of the oracle (A3), gives an upper bound on the directional derivative $\nabla V(f_t)^\top h_t$ along the full step $h_t$. Smoothness (A1) bounds the loss after the shrunk step $f_{t+1} = f_t + \eta h_t$ in terms of the same directional derivative and a quadratic term. Combining the two bounds eliminates the directional derivative and shows that each round decreases the loss by an amount proportional to $\|f_{t+1} - f_t\|^2$, up to the oracle error $\eta\epsilon_t$; summing over rounds, these decreases telescope. The argument is standard in first-order optimization. What matters in our setting is that (A1)--(A3) need to hold only round by round, along the segments visited by the algorithm, even though the matrix $B(f_t)$ changes from one round to the next.

\textbf{Approximately Non-Increasing Loss.} Since $\theta = 0$ is a feasible choice in each round's subproblem, $\min_\theta V(f_t + B(f_t)\theta) \le V(f_t)$. Combined with (A3), this gives a bound on the loss after the full step $h_t = B(f_t)\theta_t$, that is, with $\eta = 1$,
\begin{equation}
V(f_t + h_t) \le V(f_t) + \epsilon_t.
\end{equation}
To obtain a bound for the actual update $f_{t+1} = f_t + \eta h_t$, we use that $V$ is convex on the segment between $f_t$ and $f_t + h_t$, which holds by (A2). Since $\eta \in (0, 1]$,
\[
V(f_{t+1}) = V\big((1 - \eta)f_t + \eta(f_t + h_t)\big) \le (1 - \eta)V(f_t) + \eta V(f_t + h_t) \le V(f_t) + \eta\epsilon_t \le V(f_t) + \epsilon_t .
\]
This proves the first claim of the theorem.

\textbf{Per-Round Decrease.} Let $\Delta_t := f_{t+1} - f_t = \eta h_t$. By the smoothness assumption (A1), the loss at the new iterate is at most its first-order approximation plus a quadratic term,
\begin{equation}
V(f_{t+1}) \le V(f_t) + \nabla V(f_t)^\top\Delta_t + \frac{L}{2}\|\Delta_t\|^2 ,
\end{equation}
and by the strong convexity assumption (A2), the loss after the full step is at least its first-order approximation plus a quadratic term,
\begin{equation}
V(f_t + h_t) \ge V(f_t) + \nabla V(f_t)^\top h_t + \frac{m}{2}\|h_t\|^2 .
\end{equation}
Rearranging (10) gives an upper bound on the directional derivative along $h_t$,
\begin{equation}
\nabla V(f_t)^\top h_t \le V(f_t + h_t) - V(f_t) - \frac{m}{2}\|h_t\|^2 ,
\end{equation}
and bounding the difference $V(f_t + h_t) - V(f_t)$ by (8) yields
\begin{equation}
\nabla V(f_t)^\top h_t \le \epsilon_t - \frac{m}{2}\|h_t\|^2 .
\end{equation}
On the other hand, writing $\Delta_t = \eta h_t$ in (9) gives
\begin{equation}
V(f_{t+1}) \le V(f_t) + \eta\nabla V(f_t)^\top h_t + \frac{\eta^2 L}{2}\|h_t\|^2 .
\end{equation}
Substituting (12) into (13) eliminates the directional derivative, and using $\|h_t\|^2 = \|\Delta_t\|^2/\eta^2$ we obtain
\[
V(f_{t+1}) \le V(f_t) + \eta\epsilon_t - \frac{\eta m}{2}\|h_t\|^2 + \frac{\eta^2 L}{2}\|h_t\|^2 = V(f_t) + \eta\epsilon_t - \frac{m - \eta L}{2\eta}\|\Delta_t\|^2 ,
\]
or, equivalently,
\begin{equation}
\frac{m - \eta L}{2\eta}\|\Delta_t\|^2 \le V(f_t) - V(f_{t+1}) + \eta\epsilon_t .
\end{equation}
Under the step-size condition $\eta \le m/(2L)$ we have $m - \eta L \ge m/2$, so the coefficient on the left-hand side of (14) is at least $m/(4\eta)$, and
\begin{equation}
\frac{m}{4\eta}\|\Delta_t\|^2 \le V(f_t) - V(f_{t+1}) + \eta\epsilon_t .
\end{equation}

\textbf{Summing over Rounds.} Summing (15) over $t = 0, \dots, T - 1$, the loss differences telescope to $V(f_0) - V(f_T) \le V(f_0)$, since the loss is non-negative. Hence
\begin{equation}
\frac{m}{4\eta}\sum_{t=0}^{T-1}\|\Delta_t\|^2 \le V(f_0) + \eta\sum_{t=0}^{T-1}\epsilon_t .
\end{equation}
Dividing by $T$ and multiplying by $4\eta/m$ bounds the average squared gap,
\[
\frac{1}{T}\sum_{t=0}^{T-1}\|\Delta_t\|^2 \le \frac{4\eta}{m}\left(\frac{V(f_0)}{T} + \frac{\eta}{T}\sum_{t=0}^{T-1}\epsilon_t\right).
\]
Since the smallest squared gap is at most the average, taking square roots gives
\[
\min_{0\le t\le T-1}\|\Delta_t\| \le \sqrt{\frac{4\eta}{m}\left(\frac{V(f_0)}{T} + \frac{\eta}{T}\sum_{t=0}^{T-1}\epsilon_t\right)},
\]
and the bound on the minimum prediction gap stated in the theorem follows from $\sqrt{a + b} \le \sqrt a + \sqrt b$.

Finally, if $\sum_{t=0}^\infty \epsilon_t < \infty$, the right-hand side of (16) stays bounded as $T \to \infty$. Hence $\sum_{t=0}^\infty \|f_{t+1} - f_t\|^2 < \infty$, and in particular $\|f_{t+1} - f_t\| \to 0$. \qed

\textbf{Exact Oracle: A Sharper Bound.} Suppose $\epsilon_t = 0$, so that $\theta_t$ minimizes $g(\theta) := V(f_t + B(f_t)\theta)$. Then $\nabla g(\theta_t) = B(f_t)^\top\nabla V(f_t + h_t) = 0$, hence $\nabla V(f_t + h_t)^\top h_t = \theta_t^\top\nabla g(\theta_t) = 0$. The function $\tau \mapsto V(f_t + \tau h_t)$ is $m\|h_t\|^2$-strongly convex on $[0,1]$ by (A2), so its derivative increases by at least $m\|h_t\|^2$ between $\tau = 0$ and $\tau = 1$,
\[
\nabla V(f_t)^\top h_t \le \nabla V(f_t + h_t)^\top h_t - m\|h_t\|^2 = -m\|h_t\|^2 .
\]
Substituting into (13),
\[
V(f_{t+1}) \le V(f_t) - \eta\Big(m - \frac{\eta L}{2}\Big)\|h_t\|^2 = V(f_t) - \frac{m - \eta L/2}{\eta}\|\Delta_t\|^2 .
\]
If $\eta \le m/L$, then $m - \eta L/2 \ge m/2$, and summing over $t$ as before gives
\[
\sum_{t=0}^{T-1}\|f_{t+1} - f_t\|^2 \le \frac{2\eta V(f_0)}{m}, \qquad \min_{0\le t\le T-1}\|f_{t+1} - f_t\| \le \sqrt{\frac{2\eta V(f_0)}{m}}\,\frac{1}{\sqrt T}.
\]
\qed

Compared with Theorem~\ref{thm:main}, the admissible step sizes double and the constant halves; for the squared loss ($m = L = 1$), the bound holds for every $\eta \in (0,1]$.

\subsection{A Bound on the Multicalibration Error}
\emph{Overview.} Writing $n\Eh(f_t) = -B(f_t)^\top\nabla V(f_t)$, we add and subtract the gradient of the round's subproblem at $\theta_t$: that gradient is small by the (approximate) first-order optimality of the boosting subproblem, and the remainder is bounded by the prediction gap via smoothness. This is the bridge from the gap to the multicalibration error (Section 3), and it gives Eq.~(6).

\begin{lemma}\label{lem:A1}
Let the loss have a canonical gradient (Section 3), i.e., $\nabla V(f) = \phi(f) - \mathbf{y}$, with $V$ globally $L$-smooth, let $w_t = 1$ and $\eta \in (0, 1]$, and assume (A3). Then
\[
\|\Eh(f_t)\|_2 \le \frac{\|B(f_t)\|_2}{n}\left(\frac{L}{\eta}\|f_{t+1} - f_t\|_2 + \sqrt{2L\epsilon_t}\right).
\]
For the squared-error loss, $L = 1$.
\end{lemma}
\begin{proof}
Write $B = B(f_t)$ and $h_t = B\theta_t$, so $f_{t+1} - f_t = \eta h_t$. Since $\nabla V(f) = \phi(f) - \mathbf{y}$, the definition of the empirical multicalibration error (Eq.~(1)) gives $n\Eh(f_t) = B^\top(\mathbf{y} - \phi(f_t)) = -B^\top\nabla V(f_t)$. Let $g(\theta) := V(f_t + B\theta)$, the objective of the round's subproblem, so $\nabla g(\theta) = B^\top\nabla V(f_t + B\theta)$. Adding and subtracting $\nabla g(\theta_t)$,
\[
n\Eh(f_t) = -\nabla g(\theta_t) + B^\top\big(\nabla V(f_t + h_t) - \nabla V(f_t)\big).
\]
The Hessian of $g$ is $B^\top\nabla^2V(\cdot)B$, so $g$ is $L\|B\|_2^2$-smooth, and by (A3), $g(\theta_t) \le \min_\theta g(\theta) + \epsilon_t$. For a smooth function bounded below, suboptimality in value controls the gradient norm, so
\[
\|\nabla g(\theta_t)\|_2^2 \le 2L\|B\|_2^2\big(g(\theta_t) - \min_\theta g(\theta)\big) \le 2L\|B\|_2^2\,\epsilon_t .
\]
By $L$-smoothness of $V$ and sub-multiplicativity of the spectral norm,
\[
\big\|B^\top\big(\nabla V(f_t + h_t) - \nabla V(f_t)\big)\big\|_2 \le \|B\|_2\,L\,\|h_t\|_2 = \frac{L}{\eta}\|B\|_2\,\|f_{t+1} - f_t\|_2 .
\]
Combining the two bounds by the triangle inequality and dividing by $n$ proves the claim, for every $\eta \in (0,1]$. For the squared loss, $\phi = \mathrm{id}$ and $\nabla^2V = I$, so $L = 1$.
\end{proof}

\subsection{Another Bound on the Multicalibration Error}
\emph{Overview.} If a group violates multicalibration, adding one more weak learner reduces the loss by an amount controlled by the violation (Lemma~\ref{lem:A2}, via the descent lemma). With an exact oracle, that achievable decrease is in turn bounded by the squared gap between successive predictions (Proposition~\ref{prop:A3}, via convexity and the first-order optimality condition of the round).

In this subsection, we bound a group-normalized empirical multicalibration error in terms of the gap between successive predictions. The bound holds for factorized weak learners $b(x,u) = h(x)g(u)$ with group indicators $h : \R^d \to \{0, 1\}$ and interval indicators $g : \R \to \{0, 1\}$, and for losses with a canonical gradient (Section 3), $\partial L(y, f)/\partial f = \phi(f) - y$. It generalizes the bound in \citet{tax2026}, which is for the log-loss.

For any given function $\psi : \R^d \times \R \to \R$, let
\[
\E_n[\psi(X, Y)] := \frac{1}{n}\sum_{i=1}^n \psi(x_i, y_i), \qquad \E_n[\psi(X,Y) \mid h(X) = 1] := \frac{\E_n[h(X)\psi(X,Y)]}{\E_n[h(X)]}.
\]

\textbf{Assumptions on the Loss.} We assume that $L(y, f)$ is $L$-smooth in its second argument, $|L'(y, a) - L'(y, b)| \le L|a - b|$ for all $y$, and has a canonical gradient (Section 3),
\[
L'(y, f) := \frac{d}{df}L(y, f) = \phi(f) - y .
\]
These assumptions are satisfied by the squared loss ($\phi = \mathrm{id}$, $L = 1$) and by the log-loss ($\phi$ the sigmoid, $L = 1/4$). Proposition~\ref{prop:A3} additionally uses convexity of $L(y,\cdot)$, which holds for both.

\textbf{Group-Normalized Multicalibration.} For $h \in \mathcal{H}$, $g \in \mathcal{G}$ and a predictor $f$, define
\[
\Delta_{h,g}(f) := \Big|\E_n\big[h(X)\, g(f(X))\, (Y - \phi(f(X)))\big]\Big|, \qquad \tau_h(f)^2 := \E_n\big[h(X)\, v(X)\big],
\]
where $v(x) \ge 0$ is a variance proxy, $v(x) = \phi(f(x))(1 - \phi(f(x)))$ for binary classification and $v(x) = \sigma^2(x)$ for regression. We say $f$ is $\alpha$-multicalibrated with respect to $(\mathcal{H}, \mathcal{G})$ if $\Delta_{h,g}(f) \le \alpha\,\tau_h(f)$ for all $h \in \mathcal{H}$, $g \in \mathcal{G}$. Let
\[
C(f)^2 := \max_{h\in\mathcal{H},\, g\in\mathcal{G}}\ \frac{\E_n[g(f(X)) \mid h(X) = 1]}{\E_n[v(X) \mid h(X) = 1]},
\]
where the maximum is over pairs with $\E_n[h(X)g(f(X))] > 0$, and let
\[
D(f) := \E_n[L(Y, f(X))] - \inf_{\theta \in \R^p}\E_n\Big[L\Big(Y, f(X) + \sum_{j=1}^p \theta_jb_j(X, f(X))\Big)\Big]
\]
be the largest loss decrease achievable by one exact boosting step from $f$.

\begin{lemma}\label{lem:A2}
Assume the loss is $L$-smooth with a canonical gradient, and that $b = h \otimes g \in \mathcal{B}$ for all $h \in \mathcal{H}$, $g \in \mathcal{G}$. Then, for every predictor $f$, $h \in \mathcal{H}$, and $g \in \mathcal{G}$,
\[
\Delta_{h,g}(f)^2 \le 2L\,\E_n[h(X)g(f(X))]\,D(f), \qquad\text{and hence}\qquad \Delta_{h,g}(f) \le \sqrt{2L\,D(f)}\;C(f)\,\tau_h(f).
\]
In particular, $f$ is $\alpha$-multicalibrated with $\alpha = \sqrt{2L}\,C(f)\sqrt{D(f)}$.
\end{lemma}
\begin{proof}
Fix $h, g$ and let $s := \E_n[h(X)g(f(X))]$ and $c := \E_n[h(X)g(f(X))(Y - \phi(f(X)))]$, so $\Delta_{h,g}(f) = |c|$. If $s = 0$, then $c = 0$ and there is nothing to prove; otherwise let $\lambda := c/(Ls)$ and $\tilde h(x) := \lambda\,h(x)g(f(x))$, which is a multiple of the weak learner $b(x, f(x))$. By $L$-smoothness (the descent lemma) and the canonical gradient, for every $(x,y)$,
\[
L(y, f(x) + \tilde h(x)) \le L(y, f(x)) - (y - \phi(f(x)))\,\tilde h(x) + \frac{L}{2}\tilde h(x)^2 .
\]
Averaging, and using that $h$ and $g$ take values in $\{0,1\}$ so that $\E_n[\tilde h(X)^2] = \lambda^2 s$,
\[
D(f) \ge \E_n[L(Y, f(X))] - \E_n[L(Y, f(X) + \tilde h(X))] \ge \lambda c - \frac{L}{2}\lambda^2 s = \frac{c^2}{2Ls}.
\]
This is the first claim. For the second, $s/\tau_h(f)^2 = \E_n[g(f(X)) \mid h(X) = 1]/\E_n[v(X) \mid h(X) = 1] \le C(f)^2$, so $c^2 \le 2L\,D(f)\,C(f)^2\,\tau_h(f)^2$.
\end{proof}

Lemma~\ref{lem:A2} requires no condition such as $L < 1$; in particular it covers the squared loss ($L = 1$).

\begin{proposition}\label{prop:A3}
Assume the conditions of Lemma~\ref{lem:A2}: the loss is $L$-smooth with a canonical gradient, and $b = h \otimes g \in \mathcal{B}$ for all $h \in \mathcal{H}$, $g \in \mathcal{G}$. Assume, in addition, that (i) the loss is convex in its second argument, (ii) the rescaling weight at round $T$ is $w_T = 1$, and (iii) the oracle is exact at round $T$, i.e., $\theta_T$ minimizes $\theta \mapsto \E_n[L(Y, f_T(X) + \sum_j\theta_jb_j(X, f_T(X)))]$. Then $f_T$ is $\alpha$-multicalibrated with
\[
\alpha \le \sqrt{2}\,\frac{L}{\eta}\,C(f_T)\,\sqrt{\E_n\big[(f_{T+1}(X) - f_T(X))^2\big]} .
\]
\end{proposition}
\begin{proof}
Let $h_T(x) := \sum_j\theta_{T,j}b_j(x, f_T(x))$, so $f_{T+1} - f_T = \eta h_T$ and, by exactness, $D(f_T) = \E_n[L(Y, f_T(X))] - \E_n[L(Y, f_T(X) + h_T(X))]$. By convexity, $L(y, f_T(x)) - L(y, f_T(x) + h_T(x)) \le -L'(y, f_T(x))\,h_T(x)$. The first-order optimality condition of $\theta_T$ gives $\E_n[L'(Y, f_T(X) + h_T(X))\,b_j(X, f_T(X))] = 0$ for all $j$, hence $\E_n[L'(Y, f_T(X) + h_T(X))\,h_T(X)] = 0$. Therefore, by $L$-smoothness,
\begin{align*}
D(f_T) &\le \E_n\big[\big(L'(Y, f_T(X) + h_T(X)) - L'(Y, f_T(X))\big)h_T(X)\big]\\
&\le L\,\E_n[h_T(X)^2] = \frac{L}{\eta^2}\E_n\big[(f_{T+1}(X) - f_T(X))^2\big].
\end{align*}
Lemma~\ref{lem:A2} completes the proof.
\end{proof}

\textbf{Interpretation of the Bound.} The bound depends on two quantities: $C(f_T)$, which depends inversely on the average variance proxy within groups, and the squared gap between successive predictors $f_{T+1}$ and $f_T$, which vanishes as the predictions stabilize. With an inexact oracle, $D(f_T)$ additionally depends on the oracle's suboptimality, which can be reduced by increasing the number of inner boosting iterations. Note also that, since $g \le 1$,
\[
C(f)^2 \le \frac{1}{\min_{h\in\mathcal{H}}\E_n[v(X) \mid h(X) = 1]},
\]
which is finite if and only if the variance proxy has positive average within every group. For binary classification this reduces to $\E_n[\phi(f(X))(1 - \phi(f(X))) \mid h(X) = 1] > 0$ for every $h \in \mathcal{H}$, i.e., predictions are not degenerate within any group; this holds whenever $f(x) \in [-R, R]$ for some constant $R > 0$, and if $\phi(f)$ takes values in $[\epsilon, 1 - \epsilon]$ then $C(f)^2 \le 1/(\epsilon(1-\epsilon)) \le 2/\epsilon$. For regression it requires a positive variance proxy within every group.

\subsection{Proof of Theorem 4.2}
\emph{Overview.} The proof is a direct calculation. Because each round's boosting update is an orthogonal projection of the residual, the new gap $f_{t+1} - f_t$ can be written as $(1 - \eta)$ times the previous gap plus a term that involves the change $A(f_t) - A(f_{t-1})$ of the projector. The Lipschitz continuity of $A$ bounds this change by a multiple of the previous gap, which gives the one-step contraction with factor $\kappa$.

\textbf{A Property of the Exact Step.} For the squared loss with an exact oracle, the update is $f_{t+1} = f_t + \eta A(f_t)(\mathbf{y} - f_t)$, where $A(f_t)$ is the orthogonal projection onto the column space of $B(f_t)$. Applying $A(f_t)$ to the new residual and using $A(f_t)^2 = A(f_t)$ gives $A(f_t)(\mathbf{y} - f_{t+1}) = (1 - \eta)A(f_t)(\mathbf{y} - f_t)$, and multiplying by $\eta$, for every $t \ge 0$,
\begin{equation}\label{eq:A4step}
\eta A(f_t)(\mathbf{y} - f_{t+1}) = (1 - \eta)(f_{t+1} - f_t).
\end{equation}
In words, after a full step ($\eta = 1$) the new residual is orthogonal to the column space of $B(f_t)$, and after a shrunk step a fraction $1 - \eta$ of the boosting update remains.

\textbf{A Recursion for the Gap.} Let $t \ge 1$. We add and subtract $\eta A(f_{t-1})(\mathbf{y} - f_t)$ in the update, then use \eqref{eq:A4step} for round $t-1$, and finally the update of round $t-1$ in the form $\mathbf{y} - f_t = (I - \eta A(f_{t-1}))(\mathbf{y} - f_{t-1})$. This gives
\begin{align*}
f_{t+1} - f_t &= \eta A(f_t)(\mathbf{y} - f_t)\\
&= \eta(A(f_t) - A(f_{t-1}))(\mathbf{y} - f_t) + \eta A(f_{t-1})(\mathbf{y} - f_t)\\
&= \eta(A(f_t) - A(f_{t-1}))(\mathbf{y} - f_t) + (1 - \eta)(f_t - f_{t-1})\\
&= \eta(A(f_t) - A(f_{t-1}))(I - \eta A(f_{t-1}))(\mathbf{y} - f_{t-1}) + (1 - \eta)(f_t - f_{t-1}).
\end{align*}

\textbf{The Contraction.} We bound the first term using the Lipschitz continuity of $A$ along the trajectory, $\|I - \eta A(f_{t-1})\|_2 \le 1$, and $\|\mathbf{y} - f_{t-1}\|_2 \le \|\mathbf{y} - f_0\|_2$, which holds because the residual norm is non-increasing for the squared loss with an exact oracle (Theorem~\ref{thm:main}). This gives
\[
\|f_{t+1} - f_t\|_2 \le (1 - \eta + \eta L_A\|\mathbf{y} - f_0\|_2)\|f_t - f_{t-1}\|_2 = \kappa\|f_t - f_{t-1}\|_2 ,
\]
which is the claimed bound. If $\kappa < 1$, iterating gives $\|f_{t+1} - f_t\|_2 \le \kappa^t\|f_1 - f_0\|_2$, so the gap converges to zero linearly. \qed

\textbf{Hard-Split Trees.}
 In Section 6 we observe geometric decay of the prediction gap with hard-split GBDTs, although they do not satisfy the Lipschitz condition of Theorem~\ref{thm:linear}. We interpret this as evidence that large ensembles average over many independent split decisions, so the column space of $B(f_t)$ varies smoothly with $f_t$ along the trajectory even though individual trees do not, approximately recovering the smoothness regime required by Theorem~\ref{thm:linear} without literally satisfying its pointwise condition.

\subsection{Proof of Lemma 4.3}
\emph{Overview.} Lemma 4.3 follows from a perturbation bound for the projector $A(f) = B(f)B(f)^+$ in terms of the matrix $B(f)$ (Lemma~\ref{lem:A4} below), combined with the Lipschitz continuity of $B(f)$. To prove Lemma~\ref{lem:A4}, we write $A(u) - A(v)$ as a sum of three terms. In two of them the difference $B(u) - B(v)$ appears directly; the third involves the difference of the pseudoinverses of the Gram matrices, which reduces to the difference of the pseudoinverses $B(u)^+ - B(v)^+$ and is controlled by a classical perturbation bound for pseudoinverses. The singular-value bounds enter through $\|B(f)^+\|_2 \le 1/\delta$ and $\|B(f)\|_2 \le M$.

\textbf{Interpretation of the Conditions.} The conditions of Lemma~\ref{lem:LA} have a simple interpretation if $B(f_t)$ is viewed as a design matrix with $n$ data points and $p$ features. The condition that $B(f_t)$ has full column rank with smallest singular value at least $\delta > 0$ requires the weak learners to be sufficiently diverse on the training data: no weak learner can be well approximated by a linear combination of the others, uniformly along the trajectory $f_t$. The condition that the largest singular value of $B(f_t)$ is at most $M$ requires the combined output of the weak learners to remain bounded. Since
\[
\sigma_{\max}(B(f_t))^2 \le \mathrm{tr}\big(B(f_t)^\top B(f_t)\big) = \sum_{j=1}^p \|b_j(f_t)\|_2^2 ,
\]
a sufficient condition is that each weak learner is bounded on the training data, $\|b_j(f_t)\|_2 \le M/\sqrt{p}$ for all $j$.

\begin{lemma}\label{lem:A4}
Assume that for some $u, v \in \R^n$, $B(u)$ and $B(v)$ have full column rank with smallest singular values at least $\delta > 0$ and for some $M > 0$, $\|B(u)\|_2, \|B(v)\|_2 \le M$. Then,
\[
\|A(u) - A(v)\|_2 \le \frac{2}{\delta}\left(1 + c\,\frac{M^2}{\delta^2}\right)\|B(u) - B(v)\|_2
\]
where $c = \sqrt{2}$.
\end{lemma}
\begin{proof}
Write $B_u = B(u)$, $B_v = B(v)$, $A_u = A(u) = B_uB_u^+$ and $A_v = A(v) = B_vB_v^+$. Throughout, we use two facts about a matrix $B$ with full column rank and smallest singular value at least $\delta$. First, $\|B^+\|_2 = 1/\sigma_{\min}(B) \le 1/\delta$. Second, $B^\top B$ is invertible and $B^+ = (B^\top B)^{-1}B^\top$, which gives the identities
\[
B^+ = (B^\top B)^+B^\top, \qquad (B^+)^\top = B(B^\top B)^+, \qquad (B^\top B)^+ = B^+(B^+)^\top .
\]

\textbf{Decomposing the Difference of Projectors.} Adding and subtracting $B_vB_u^+$,
\[
A_u - A_v = (B_u - B_v)B_u^+ + B_v(B_u^+ - B_v^+).
\]
To make the difference $B_u - B_v$ appear also in the second term, we express the pseudoinverses through the Gram matrices. By the first identity, and adding and subtracting $(B_v^\top B_v)^+B_u^\top$,
\[
B_u^+ - B_v^+ = \big((B_u^\top B_u)^+ - (B_v^\top B_v)^+\big)B_u^\top + (B_v^\top B_v)^+(B_u - B_v)^\top .
\]
Multiplying from the left by $B_v$ and using the second identity, $B_v(B_v^\top B_v)^+ = (B_v^+)^\top$, we arrive at
\[
A_u - A_v = (B_u - B_v)B_u^+ + B_v\big((B_u^\top B_u)^+ - (B_v^\top B_v)^+\big)B_u^\top + (B_v^+)^\top(B_u - B_v)^\top .
\]

\textbf{Bounding the Three Terms.} Since $\|B_u^+\|_2, \|B_v^+\|_2 \le 1/\delta$, the first and the last term each have norm at most $\|B_u - B_v\|_2/\delta$. Since $\|B_u\|_2, \|B_v\|_2 \le M$, the middle term has norm at most $M^2\|(B_u^\top B_u)^+ - (B_v^\top B_v)^+\|_2$. Hence
\[
\|A_u - A_v\|_2 \le \frac{2}{\delta}\|B_u - B_v\|_2 + M^2\|(B_u^\top B_u)^+ - (B_v^\top B_v)^+\|_2 .
\]
It remains to bound the difference of the pseudoinverses of the Gram matrices. By the third identity, and adding and subtracting $B_u^+(B_v^+)^\top$,
\[
(B_u^\top B_u)^+ - (B_v^\top B_v)^+ = B_u^+(B_u^+ - B_v^+)^\top + (B_u^+ - B_v^+)(B_v^+)^\top ,
\]
whose norm is at most $(2/\delta)\|B_u^+ - B_v^+\|_2$. Therefore,
\[
\|A_u - A_v\|_2 \le \frac{2}{\delta}\big(\|B_u - B_v\|_2 + M^2\|B_u^+ - B_v^+\|_2\big).
\]

\textbf{Bounding the Difference of Pseudoinverses.} We use a classical perturbation bound for the Moore--Penrose pseudoinverse of matrices of equal rank, due to \citet{wedin1973} \citep[see][Theorem 3.4]{stewart1977}: for two matrices of the same size and the same rank, the distance between their pseudoinverses is at most a constant times the distance between the matrices, multiplied by the norms of the two pseudoinverses. For the spectral norm and matrices of full column rank, the constant is $c = \sqrt{2}$ (and $c = 1$ for square matrices). In other words, on matrices of fixed rank whose smallest nonzero singular value is bounded away from zero, the pseudoinverse is a Lipschitz function. Here $B_u$ and $B_v$ both have full column rank $p$, so
\[
\|B_u^+ - B_v^+\|_2 \le c\,\|B_u^+\|_2\,\|B_v^+\|_2\,\|B_u - B_v\|_2 \le \frac{c}{\delta^2}\|B_u - B_v\|_2 .
\]
Substituting this bound into the previous display gives
\[
\|A_u - A_v\|_2 \le \frac{2}{\delta}\left(1 + \frac{cM^2}{\delta^2}\right)\|B_u - B_v\|_2 ,
\]
which proves the lemma.
\end{proof}

\begin{proof}[Proof of Lemma~\ref{lem:LA}]
Let $u, v \in \mathcal{F}$. The assumptions of Lemma~\ref{lem:LA} imply those of Lemma~\ref{lem:A4}, and the Lipschitz continuity of $B(f)$ on $\mathcal{F}$ gives $\|B(u) - B(v)\|_2 \le L_B\|u - v\|_2$. Hence, by Lemma~\ref{lem:A4},
\[
\|A(u) - A(v)\|_2 \le \frac{2}{\delta}\left(1 + c\,\frac{M^2}{\delta^2}\right)L_B\,\|u - v\|_2 ,
\]
which is the claimed bound on $L_A$.
\end{proof}

\subsection{Proof of Lemma 4.4}
\emph{Overview.} For factorized weak learners, every row of $B(f)$ is a Kronecker product $h(x_i) \otimes g(f_i)$. This allows us to write $B(f)$ as the product of a block-diagonal matrix built from the vectors $g(f_i)$ and a matrix that does not depend on $f$. Reordering the columns of $B(f)$ gives an analogous factorization with the roles of $h$ and $g$ exchanged, in which the dependence on $f$ is isolated in the single matrix $G(f) = (g(f_1), \dots, g(f_n))^\top$. The Lipschitz constant of $B(f)$ is then the Lipschitz constant of $G(f)$, which follows from that of the functions $g$, multiplied by the size of the features $h(x_i)$. The same factorizations also bound $\|B(f)\|_2$, which provides the constant $M$ in Lemma~\ref{lem:LA}.

\textbf{Setup and Notation.} Recall from Section 3 that for factorized weak learners $b(x, u) = h(x)g(u)$, with $h(x) = (h_1(x), \dots, h_m(x))^\top$ and $g(u) = (g_1(u), \dots, g_k(u))^\top$, the $i$-th row of $B(f)$ is $\big(h(x_i) \otimes g(f_i)\big)^\top$ and $p = mk$; see (2) and (3). Let
\[
H(X) = (h(x_1), \dots , h(x_n))^\top \in \R^{n\times m} \quad\text{and}\quad G(f) = (g(f_1), \dots , g(f_n))^\top \in \R^{n\times k},
\]
and, for a matrix with rows $z_1^\top, \dots, z_n^\top$, let $\|\cdot\|_{2,\infty}$ denote the largest Euclidean norm of a row; for instance, $\|H(X)\|_{2,\infty} = \max_i\|h(x_i)\|_2$. We also use the block-diagonal matrices
\begin{align*}
D_g(f) &= \text{block-diag}\big(g(f_1)^\top, \dots , g(f_n)^\top\big) \in \R^{n\times nk},\\
D_h(X) &= \text{block-diag}\big(h(x_1)^\top, \dots , h(x_n)^\top\big) \in \R^{n\times nm}.
\end{align*}
Since $D_g(f)D_g(f)^\top = \mathrm{diag}(\|g(f_1)\|_2^2, \dots, \|g(f_n)\|_2^2)$, the spectral norm of such a matrix is the largest norm of its blocks: $\|D_g(f)\|_2 = \|G(f)\|_{2,\infty}$, and likewise $\|D_h(X)\|_2 = \|H(X)\|_{2,\infty}$.

\textbf{Two Factorizations of $B(f)$.} By the mixed-product property of the Kronecker product, $g(f_i)^\top\big(h(x_i)^\top \otimes I_k\big) = h(x_i)^\top \otimes g(f_i)^\top$, which is the $i$-th row of $B(f)$. Hence
\[
B(f) = D_g(f)\,\big(H(X) \otimes I_k\big).
\]
In the same way, $h(x_i)^\top\big(g(f_i)^\top \otimes I_m\big) = g(f_i)^\top \otimes h(x_i)^\top$, which is the $i$-th row of $B(f)$ with its entries reordered. Hence the matrix
\[
R(f) := D_h(X)\,\big(G(f) \otimes I_m\big)
\]
satisfies $R(f) = B(f)P$ for a permutation matrix $P$ that depends only on $m$ and $k$, not on $f$. Since multiplication by a permutation matrix does not change the spectral norm, $\|R(f)\|_2 = \|B(f)\|_2$ and $\|R(f) - R(f')\|_2 = \|B(f) - B(f')\|_2$ for all $f, f' \in \R^n$.

\textbf{Bounds on the Norm of $B(f)$.} Using $\|H(X) \otimes I_k\|_2 = \|H(X)\|_2$ and $\|G(f) \otimes I_m\|_2 = \|G(f)\|_2$, the two factorizations give
\[
\|B(f)\|_2 \le \|G(f)\|_{2,\infty}\|H(X)\|_2 \qquad\text{and}\qquad \|B(f)\|_2 \le \|H(X)\|_{2,\infty}\|G(f)\|_2 .
\]
Either bound can serve as the constant $M$ in Lemma~\ref{lem:LA}.

\textbf{Lipschitz Continuity of $B(f)$.} In the second factorization, only $G(f)$ depends on $f$, so
\[
R(f) - R(f') = D_h(X)\big((G(f) - G(f')) \otimes I_m\big),
\]
and therefore
\[
\|B(f) - B(f')\|_2 = \|R(f) - R(f')\|_2 \le \|H(X)\|_{2,\infty}\|G(f) - G(f')\|_2 .
\]
It remains to bound the change in $G(f)$. If every $g_j$ is $L_G$-Lipschitz, then, bounding the spectral norm by the Frobenius norm,
\begin{align*}
\|G(f) - G(f')\|_2 &\le \|G(f) - G(f')\|_F = \sqrt{\sum_{j=1}^k\sum_{i=1}^n\big(g_j(f_i) - g_j(f_i')\big)^2}\\
&\le \sqrt{\sum_{j=1}^k\sum_{i=1}^n L_G^2(f_i - f_i')^2} = L_G\sqrt{k}\,\|f - f'\|_2 .
\end{align*}
Combining the last two displays,
\[
\|B(f) - B(f')\|_2 \le \Big(\max_i\|h(x_i)\|_2\Big)\sqrt{k}\,L_G\,\|f - f'\|_2 ,
\]
that is, $L_B \le (\max_i\|h(x_i)\|_2)\sqrt{k}\,L_G$, as claimed. Under the conditions of Lemma~\ref{lem:LA}, $A(f)$ is then Lipschitz as well, with the constant $L_A$ given there. \qed

\subsection{Proof of Theorem 4.5}
\emph{Overview.} The proof follows the same pattern as for unit weights, with the rescaling treated as a perturbation. We track the training loss $V(f) = \frac12\|\mathbf{y} - f\|_2^2$ along the iterates. As for unit weights, each round decreases it by a multiple of $\|A(f_t)(\mathbf{y} - f_t)\|_2^2$, but now up to a perturbation $\xi_t$ caused by the rescaling. Because the residuals stay bounded, $|\xi_t|$ is at most a constant times $1 - w_t$, so the perturbations are summable under $\sum_t(1 - w_t) < \infty$. This gives $A(f_t)(\mathbf{y} - f_t) \to 0$, hence convergence of the gap and, by a Cauchy-sequence argument, convergence of the loss. Asymptotic multicalibration then follows from the first-order optimality condition of each round's subproblem. For the rate, we modify $V$ by subtracting the accumulated effect of the rescaling, which yields a function that is non-increasing along the iterates; summing its decreases and applying the Cauchy--Schwarz inequality bounds $\sum_t\|f_{t+1} - f_t\|_2^2$.

\textbf{Discussion.} The theorem shows that relaxing the weights for regularization does not break the convergence results, provided the weights $w_t \in (0,1]$ converge to 1 fast enough that $\sum_t(1 - w_t) < \infty$. Without this condition the conclusions can fail: for instance, a constant weight $w_t = w < 1$ shrinks the predictor towards zero in every round, so the iterates need not approach a multicalibrated predictor.

The non-increasing loss property of Theorem~\ref{thm:main} need not hold in the relaxed case. A weaker property does hold: the proof below shows that $\|\mathbf{y} - f_t\|_2 \le \rho_t\|\mathbf{y} - f_0\|_2 + (1 - \rho_t)\|\mathbf{y}\|_2$, where $\rho_t := \prod_{s=0}^{t-1} w_s$. For unit weights, $\rho_t \equiv 1$, and the bound says that the residual never exceeds its initial value. Under $\sum_t(1 - w_t) < \infty$, the products $\rho_t$ decrease to a positive limit $1 - \epsilon$ with $\epsilon \in [0, 1)$, so that $\|\mathbf{y} - f_t\|_2 \le \|\mathbf{y} - f_0\|_2 + \epsilon\max\{0, \|\mathbf{y}\|_2 - \|\mathbf{y} - f_0\|_2\}$ for all $t$: the training loss can exceed its initial value only by an amount of order $\epsilon$.

In the convergence rate bound, $\gamma$ is increasing in $C_w := \sum_t(1 - w_t)$ and in $\tilde C_w := \sum_t(1 - w_t)^2$, and $\gamma = 0$ for unit weights. In that case, the bound coincides with the exact-oracle bound of Theorem~\ref{thm:main} for the squared loss (Appendix A.1), $\sqrt{\eta}\|\mathbf{y} - f_0\|_2/\sqrt{T}$.

\textbf{Setting.} Throughout, the loss is the squared loss, the oracle is exact, and the update (5) reads $f_{t+1} = w_t\big(f_t + \eta A(f_t)(\mathbf{y} - f_t)\big)$ with $\eta \in (0,1]$. We repeatedly use that $A(f_t)$ is an orthogonal projection, so that $\|I - \eta A(f_t)\|_2 \le 1$ and $(\mathbf{y} - f_t)^\top A(f_t)(\mathbf{y} - f_t) = \|A(f_t)(\mathbf{y} - f_t)\|_2^2$.

\textbf{Residual Recursion and the Training Loss.} Subtracting (5) from $\mathbf{y}$ gives the recursion for the residual,
\begin{equation}\label{eq:A7res}
\mathbf{y} - f_{t+1} = (1 - w_t)\mathbf{y} + w_t(I - \eta A(f_t))(\mathbf{y} - f_t).
\end{equation}
Expanding the square of \eqref{eq:A7res}, and using $\|(I - \eta A(f_t))(\mathbf{y} - f_t)\|_2^2 = \|\mathbf{y} - f_t\|_2^2 - \eta(2 - \eta)\|A(f_t)(\mathbf{y} - f_t)\|_2^2$, which holds because $A(f_t)$ is an orthogonal projection, we obtain the change of the training loss in one round,
\begin{equation}\label{eq:A7dV}
V(f_{t+1}) - V(f_t) = -w_t^2\eta\left(1 - \frac{1}{2}\eta\right)\|A(f_t)(\mathbf{y} - f_t)\|_2^2 + \xi_t ,
\end{equation}
where the perturbation due to the rescaling is
\[
\xi_t = -\frac{1}{2}(1 - w_t^2)\|\mathbf{y} - f_t\|_2^2 + \frac{1}{2}(1 - w_t)^2\|\mathbf{y}\|_2^2 + w_t(1 - w_t)\,\mathbf{y}^\top(I - \eta A(f_t))(\mathbf{y} - f_t).
\]
For unit weights, $\xi_t = 0$, and \eqref{eq:A7dV} shows that the loss decreases in every round.

\textbf{Bounded Residuals and Summable Perturbations.} Taking norms in \eqref{eq:A7res} and using $\|I - \eta A(f_t)\|_2 \le 1$,
\[
\|\mathbf{y} - f_{t+1}\|_2 \le (1 - w_t)\|\mathbf{y}\|_2 + w_t\|\mathbf{y} - f_t\|_2 .
\]
Since $w_t \in (0,1]$, the right-hand side is a convex combination, and by induction $\|\mathbf{y} - f_t\|_2 \le \rho_t\|\mathbf{y} - f_0\|_2 + (1 - \rho_t)\|\mathbf{y}\|_2$ with $\rho_t = \prod_{s<t}w_s$. In particular, for all $t \ge 0$,
\[
\|\mathbf{y} - f_t\|_2 \le R := \max\{\|\mathbf{y} - f_0\|_2, \|\mathbf{y}\|_2\}, \qquad \|f_t\|_2 \le R_f := \|\mathbf{y}\|_2 + R .
\]
Each of the three terms of $\xi_t$ is therefore at most a constant times $1 - w_t$: using $1 - w_t^2 \le 2(1 - w_t)$, $(1 - w_t)^2 \le 1 - w_t$, and, by the Cauchy--Schwarz inequality, $|\mathbf{y}^\top(I - \eta A(f_t))(\mathbf{y} - f_t)| \le \|\mathbf{y}\|_2R$, we get
\[
|\xi_t| \le C_\xi(1 - w_t), \qquad C_\xi := R^2 + \tfrac12\|\mathbf{y}\|_2^2 + \|\mathbf{y}\|_2R .
\]
Hence $\sum_{t=0}^\infty|\xi_t| \le C_\xi C_w < \infty$.

\textbf{Convergence of the Gap and of the Loss.} Summing \eqref{eq:A7dV} over $t = 0, \dots, T-1$ and using $V(f_T) \ge 0$,
\[
\eta\left(1 - \frac{1}{2}\eta\right)\sum_{t=0}^{T-1}w_t^2\|A(f_t)(\mathbf{y} - f_t)\|_2^2 \le V(f_0) + \sum_{t=0}^{T-1}|\xi_t| .
\]
The right-hand side is bounded uniformly in $T$, so the series $\sum_t w_t^2\|A(f_t)(\mathbf{y} - f_t)\|_2^2$ converges. Since $w_t \to 1$, it follows that $A(f_t)(\mathbf{y} - f_t) \to 0$. The difference of successive predictions is
\[
f_{t+1} - f_t = w_t\eta A(f_t)(\mathbf{y} - f_t) - (1 - w_t)f_t ,
\]
so $\|f_{t+1} - f_t\|_2 \le w_t\eta\|A(f_t)(\mathbf{y} - f_t)\|_2 + (1 - w_t)\|f_t\|_2$. Both terms tend to zero, the first as just shown and the second because $\|f_t\|_2 \le R_f$ and $w_t \to 1$. Hence $f_{t+1} - f_t \to 0$.

To show that the loss converges, let $s < t$. Summing \eqref{eq:A7dV} from $s$ to $t-1$,
\[
|V(f_t) - V(f_s)| \le \eta\left(1 - \frac{1}{2}\eta\right)\sum_{i=s}^{t-1}w_i^2\|A(f_i)(\mathbf{y} - f_i)\|_2^2 + \sum_{i=s}^{t-1}|\xi_i| .
\]
Both series on the right-hand side converge, so their tails tend to zero, and $\{V(f_t)\}$ is a Cauchy sequence. Hence $V(f_t)$ converges to a limit $V^*$. This proves part (1) of the theorem.

\textbf{Asymptotic Multicalibration.} Since the oracle is exact, the update (5) can be written as
\begin{equation}\label{eq:A7upd}
f_{t+1} = w_t\big(f_t + \eta B(f_t)\theta(f_t)\big),
\end{equation}
where $\theta(f_t)$ minimizes the round's subproblem $L_t(\theta) = \frac{1}{2}\|\mathbf{y} - f_t - B(f_t)\theta\|_2^2$. The first-order optimality condition of this subproblem is
\begin{equation}\label{eq:A7foc}
B(f_t)^\top\big(\mathbf{y} - f_t - B(f_t)\theta(f_t)\big) = 0.
\end{equation}
By \eqref{eq:A7foc}, $n\Eh(f_t) = B(f_t)^\top(\mathbf{y} - f_t) = B(f_t)^\top B(f_t)\theta(f_t)$, and by \eqref{eq:A7upd}, $B(f_t)\theta(f_t) = \big(f_{t+1} - f_t + (1 - w_t)f_t\big)/(\eta w_t)$. Therefore, for any weights $0 < w_t \le 1$,
\[
\Eh(f_t) = \frac{1 - w_t}{\eta w_t n}B(f_t)^\top f_t + \frac{1}{\eta w_t n}B(f_t)^\top(f_{t+1} - f_t).
\]
If $\|B(f_t)\|_2$ is bounded along the trajectory then, since also $\|f_t\|_2 \le R_f$, there are constants $C_1, C_2 > 0$ such that
\[
\|\Eh(f_t)\|_2 \le C_1\frac{1 - w_t}{w_t} + C_2\frac{1}{w_t}\|f_{t+1} - f_t\|_2 .
\]
Since $w_t \to 1$ and $f_{t+1} - f_t \to 0$, both terms tend to zero, so $\Eh(f_t) \to 0$. This proves part (2).

\textbf{A Non-Increasing Lyapunov Function.} For the rate, we need a quantity that decreases in every round. We first record an identity for the change of the squared residual norm.

\begin{lemma}\label{lem:A5}
The difference of the successive residual norms satisfies, for every $t \ge 0$,
\[
\|\mathbf{y} - f_t\|_2^2 - \|\mathbf{y} - f_{t+1}\|_2^2 = w_t\eta(2 - w_t\eta)\|A(f_t)(\mathbf{y} - f_t)\|_2^2 + (1 - w_t)^2\|f_t\|_2^2 - 2(1 - w_t)f_t^\top(\mathbf{y} - f_{t+1}).
\]
\end{lemma}
\begin{proof}
Rearranging (5), $\mathbf{y} - f_{t+1} = (I - w_t\eta A(f_t))(\mathbf{y} - f_t) + (1 - w_t)f_t$. Expanding the square,
\[
\|\mathbf{y} - f_{t+1}\|_2^2 = \|(I - w_t\eta A(f_t))(\mathbf{y} - f_t)\|_2^2 + 2(1 - w_t)f_t^\top(I - w_t\eta A(f_t))(\mathbf{y} - f_t) + (1 - w_t)^2\|f_t\|_2^2 .
\]
Since $A(f_t)$ is an orthogonal projection, the first term equals $\|\mathbf{y} - f_t\|_2^2 - w_t\eta(2 - w_t\eta)\|A(f_t)(\mathbf{y} - f_t)\|_2^2$. In the cross term, we substitute $(I - w_t\eta A(f_t))(\mathbf{y} - f_t) = (\mathbf{y} - f_{t+1}) - (1 - w_t)f_t$, which turns it into $2(1 - w_t)f_t^\top(\mathbf{y} - f_{t+1}) - 2(1 - w_t)^2\|f_t\|_2^2$. Collecting terms proves the identity.
\end{proof}

The last term in Lemma~\ref{lem:A5} has no definite sign, which is why the loss need not decrease. We therefore add the accumulated cross terms back and consider
\[
U_t := \frac{1}{2}\|\mathbf{y} - f_t\|_2^2 - \sum_{s=0}^{t-1}(1 - w_s)f_s^\top(\mathbf{y} - f_{s+1}) + C ,
\]
where the constant $C \ge 0$ is chosen below so that $U_t \ge 0$. By Lemma~\ref{lem:A5}, the cross terms cancel in $U_t - U_{t+1}$, and
\[
U_t - U_{t+1} = w_t\eta\left(1 - \frac{1}{2}w_t\eta\right)\|A(f_t)(\mathbf{y} - f_t)\|_2^2 + \frac{1}{2}(1 - w_t)^2\|f_t\|_2^2 \ge 0 ,
\]
so $U_t$ is non-increasing along the iterates.

To choose $C$, it suffices that every cross term is bounded,
\begin{equation}\label{eq:A7cross}
|f_t^\top(\mathbf{y} - f_{t+1})| \le D \quad \text{for all } t \ge 0,
\end{equation}
since then $\big|\sum_{s=0}^{t-1}(1 - w_s)f_s^\top(\mathbf{y} - f_{s+1})\big| \le DC_w$, and $C := DC_w$ makes $U_t \ge 0$. To verify \eqref{eq:A7cross}, we use the Cauchy--Schwarz inequality, $\|f_t\|_2 \le \|\mathbf{y}\|_2 + \|\mathbf{y} - f_t\|_2$, and the residual bound $R$,
\begin{equation}\label{eq:A7cs}
|f_t^\top(\mathbf{y} - f_{t+1})| \le (\|\mathbf{y}\|_2 + \|\mathbf{y} - f_t\|_2)\|\mathbf{y} - f_{t+1}\|_2 \le (\|\mathbf{y}\|_2 + R)R .
\end{equation}
Hence \eqref{eq:A7cross} holds with $D := (\|\mathbf{y}\|_2 + R)R = \max\{2\|\mathbf{y}\|_2^2, (\|\mathbf{y}\|_2 + \|\mathbf{y} - f_0\|_2)\|\mathbf{y} - f_0\|_2\}$.

\textbf{Convergence Rate.} We split each step into a boosting part and a rescaling part,
\[
f_{t+1} - f_t = a_t + b_t, \qquad a_t := w_t\eta A(f_t)(\mathbf{y} - f_t), \qquad b_t := -(1 - w_t)f_t .
\]
The boosting part is controlled by the decrease of $U_t$. Since $w_t\eta \le \eta \le 1$, we have $w_t\eta/(1 - w_t\eta/2) \le 2\eta$, and therefore
\[
\|a_t\|_2^2 = \frac{w_t\eta}{1 - w_t\eta/2}\cdot w_t\eta\left(1 - \frac{1}{2}w_t\eta\right)\|A(f_t)(\mathbf{y} - f_t)\|_2^2 \le 2\eta(U_t - U_{t+1}).
\]
Summing, and using $U_T \ge 0$ and $U_0 = \frac12\|\mathbf{y} - f_0\|_2^2 + DC_w$,
\[
\sum_{t=0}^{T-1}\|a_t\|_2^2 \le 2\eta U_0 = \eta\|\mathbf{y} - f_0\|_2^2 + 2\eta DC_w .
\]
The rescaling part is controlled by the summability of the weights: $\|b_t\|_2 \le (1 - w_t)R_f$, so
\[
\sum_{t=0}^{T-1}\|b_t\|_2^2 \le R_f^2\sum_{t=0}^{\infty}(1 - w_t)^2 = R_f^2\tilde C_w , \qquad \tilde C_w \le C_w .
\]
Finally, expanding $\|a_t + b_t\|_2^2$ and applying the Cauchy--Schwarz inequality to the sum of the cross terms,
\[
\sum_{t=0}^{T-1}\|a_t + b_t\|_2^2 \le \sum_{t=0}^{T-1}\|a_t\|_2^2 + 2\sqrt{\sum_{t=0}^{T-1}\|a_t\|_2^2}\sqrt{\sum_{t=0}^{T-1}\|b_t\|_2^2} + \sum_{t=0}^{T-1}\|b_t\|_2^2 .
\]
Combining the three bounds,
\begin{align*}
\sum_{t=0}^{T-1}\|f_{t+1} - f_t\|_2^2 &\le \eta\|\mathbf{y} - f_0\|_2^2 + 2\eta DC_w\\
&\quad + 2\sqrt{\eta\|\mathbf{y} - f_0\|_2^2 + 2\eta DC_w}\,\sqrt{R_f^2\tilde C_w} + R_f^2\tilde C_w =: K .
\end{align*}
Since the smallest squared gap is at most the average,
\[
\min_{0\le t\le T-1}\|f_{t+1} - f_t\|_2 \le \sqrt{K}\,\frac{1}{\sqrt{T}},
\]
and $\sqrt{a + b} \le \sqrt a + \sqrt b$ gives $\sqrt{K} \le \sqrt{\eta}\|\mathbf{y} - f_0\|_2 + \gamma$ with
\[
\gamma^2 = 2\eta DC_w + 2\sqrt{\eta\|\mathbf{y} - f_0\|_2^2 + 2\eta DC_w}\,\sqrt{R_f^2\tilde C_w} + R_f^2\tilde C_w .
\]
Since $\gamma$ is increasing in $\tilde C_w$ and $\tilde C_w \le C_w$, $\gamma$ can be bounded in terms of $\eta$, $\mathbf{y}$, $f_0$ and $C_w = \sum_t(1 - w_t)$ alone, as stated in the theorem. This proves part (3). \qed

\subsection{Proof of Theorem 4.6}
\emph{Overview.} Monotonicity is immediate because $\omega(f_t)$ minimizes the residual over the line spanned by $\varphi(f_t)$. Comparing the adaptive step with the unscaled step $\varphi(f_t)$ gives a per-round decrease of the squared residual by $\eta(2-\eta)\|A(f_t)(\mathbf{y} - f_t)\|_2^2$, which telescopes and yields asymptotic multicalibration with an explicit rate, without any continuity assumption on $A(\cdot)$. Non-negativity of $\omega(f_t)$ for $t \ge 1$ follows from the fact that $f_t$ is then an orthogonal projection of $\mathbf{y}$, together with the inequality between the arithmetic and geometric means (AM--GM), $ab \le \frac12(a^2 + b^2)$ for $a, b \ge 0$.

Throughout, write $r_t := \mathbf{y} - f_t$ and $A_t := A(f_t)$. If $\varphi(f_t) = 0$, every $\omega$ gives the same point $\omega\varphi(f_t) = 0$; we then set $f_{t+1} = 0$ (and $\omega(f_t) := 0$).

\textbf{Non-Increasing Loss.} By definition, $\varphi(f_t) = f_t + \eta A(f_t)(\mathbf{y} - f_t) = f_t + \eta B(f_t)\theta(f_t)$, where $\theta(f_t)$ is the minimizer of $L_t(\theta) = \frac{1}{2}\|\mathbf{y} - f_t - B(f_t)\theta\|_2^2$ over $\theta \in \R^p$. Hence,
\[
\|\mathbf{y} - \varphi(f_t)\|_2 \le \|\mathbf{y} - f_t\|_2.
\]
Since $\omega = 1$ is a feasible choice,
\[
\min_\omega\|\mathbf{y} - \omega\varphi(f_t)\|_2 \le \|\mathbf{y} - \varphi(f_t)\|_2 ,
\]
and, by definition of $\omega(f_t)$, $\|\mathbf{y} - \omega(f_t)\varphi(f_t)\|_2 = \min_\omega\|\mathbf{y} - \omega\varphi(f_t)\|_2$. Since $f_{t+1} = \omega(f_t)\varphi(f_t)$, it follows that
\[
\|\mathbf{y} - f_{t+1}\|_2 \le \|\mathbf{y} - f_t\|_2.
\]
(If $\varphi(f_t) = 0$, then $f_{t+1} = 0 = \varphi(f_t)$ and the same inequality holds.)

\textbf{Multicalibration Error: Finite-Time Bound and Convergence.} Since $\mathbf{y} - \varphi(f_t) = (I - \eta A_t)r_t$ and $A_t$ is an orthogonal projection,
\[
\|\mathbf{y} - \varphi(f_t)\|_2^2 = \|r_t\|_2^2 - 2\eta r_t^\top A_tr_t + \eta^2\|A_tr_t\|_2^2 = \|r_t\|_2^2 - \eta(2 - \eta)\|A_tr_t\|_2^2 .
\]
Combining with $\|r_{t+1}\|_2 = \min_\omega\|\mathbf{y} - \omega\varphi(f_t)\|_2 \le \|\mathbf{y} - \varphi(f_t)\|_2$ (which also holds when $\varphi(f_t) = 0$) gives, for every $t \ge 0$,
\begin{equation}
\|r_{t+1}\|_2^2 \le \|r_t\|_2^2 - \eta(2 - \eta)\|A(f_t)r_t\|_2^2 .
\end{equation}
Summing (24) over $t = 0, \dots, T-1$ and using $\|r_T\|_2 \ge 0$,
\[
\sum_{t=0}^{T-1}\|A(f_t)(\mathbf{y} - f_t)\|_2^2 \le \frac{\|\mathbf{y} - f_0\|_2^2}{\eta(2 - \eta)},
\]
so $A(f_t)(\mathbf{y} - f_t) \to 0$ and $\min_{0\le t\le T-1}\|A(f_t)(\mathbf{y} - f_t)\|_2 \le \|\mathbf{y} - f_0\|_2/\sqrt{\eta(2-\eta)T}$. Finally, $B(f)^\top A(f) = B(f)^\top$ (as $A(f) = B(f)B(f)^+$ is the orthogonal projector onto the column space of $B(f)$), so
\[
\|\Eh(f_t)\|_2 = \frac{1}{n}\|B(f_t)^\top A(f_t)(\mathbf{y} - f_t)\|_2 \le \frac{\|B(f_t)\|_2}{n}\|A(f_t)(\mathbf{y} - f_t)\|_2 \to 0
\]
whenever $\|B(f_t)\|_2$ is bounded along the trajectory (e.g., for bounded weak learners, $\|B(f)\|_2 \le \sqrt{np}\max_{j,x,u}|b_j(x,u)|$). If $M_B := \sup_t\|B(f_t)\|_2 < \infty$, evaluating this bound at a round $t' \le T - 1$ that attains $\min_{0\le t\le T-1}\|A(f_t)(\mathbf{y} - f_t)\|_2$ gives
\[
\min_{0\le t\le T-1}\|\Eh(f_t)\|_2 \le \|\Eh(f_{t'})\|_2 \le \frac{M_B}{n}\,\frac{\|\mathbf{y} - f_0\|_2}{\sqrt{\eta(2-\eta)T}} .
\]

\textbf{Non-Negative Rescaling.} Let $t \ge 1$. If $\varphi(f_{t-1}) \neq 0$, then $f_t = \omega(f_{t-1})\varphi(f_{t-1})$ is the orthogonal projection of $\mathbf{y}$ onto the span of $\varphi(f_{t-1})$; if $\varphi(f_{t-1}) = 0$, then $f_t = 0$. In both cases $f_t^\top r_t = 0$, hence $\mathbf{y}^\top f_t = \|f_t\|_2^2$ and
\[
\mathbf{y}^\top\varphi(f_t) = \mathbf{y}^\top f_t + \eta(f_t + r_t)^\top A_tr_t = \|f_t\|_2^2 + \eta\|A_tr_t\|_2^2 + \eta f_t^\top A_tr_t .
\]
By the Cauchy--Schwarz inequality, $|f_t^\top A_tr_t| = |(A_tf_t)^\top A_tr_t| \le \|f_t\|_2\|A_tr_t\|_2$, and by the AM--GM inequality, $\|f_t\|_2\|A_tr_t\|_2 \le \frac12(\|f_t\|_2^2 + \|A_tr_t\|_2^2)$. Therefore
\[
\mathbf{y}^\top\varphi(f_t) \ge \Big(1 - \frac{\eta}{2}\Big)\|f_t\|_2^2 + \frac{\eta}{2}\|A_tr_t\|_2^2 \ge 0
\]
for any $\eta \in (0, 2]$ (in particular for $\eta \in (0, 1]$, the range used throughout the paper), so $\omega(f_t) = \mathbf{y}^\top\varphi(f_t)/\|\varphi(f_t)\|_2^2 \ge 0$ whenever $\varphi(f_t) \neq 0$. The same inequality shows that $\varphi(f_t) = 0$ with $t \ge 1$ is possible only if $f_t = 0$ and $A_t\mathbf{y} = 0$, i.e., only at an empirically multicalibrated fixed point. For $t = 0$ the claim can fail: with $f_0 = -\mathbf{y}$ and $A(f_0) = 0$, $\omega(f_0) = -1$. \qed

For completeness we record the spectral structure of the matrix behind this argument. For $t \ge 1$ with $\varphi(f_{t-1}) \neq 0$, write $M(f_{t-1}) = \tilde\varphi\tilde\varphi^\top$ with $\tilde\varphi = \varphi(f_{t-1})/\|\varphi(f_{t-1})\|_2$; then $f_t = M(f_{t-1})\mathbf{y}$ and $\varphi(f_t) = N(f_t)\mathbf{y}$ with $N(f_t) := \eta A_t + (I - \eta A_t)M(f_{t-1})$. The argument above does not need the following lemma.

\begin{lemma}\label{lem:A6}
Let $A$ be an orthogonal projection matrix of rank $r \ge 1$ and $M = vv^\top$ with $\|v\| = 1$, $Av \neq 0$ and $(I - A)v \neq 0$. Then the eigenvalues of
\[
\eta A + (I - \eta A)M
\]
are the value 1 with multiplicity 1, $\eta\|(I - A)v\|_2^2$ with multiplicity 1, $\eta$ with multiplicity $r - 1$, and the value 0 with multiplicity $n - r - 1$ (some of these values may coincide).
\end{lemma}
\begin{proof}
Let $N = \eta A + (I - \eta A)vv^\top$. Then $Nv = \eta Av + (I - \eta A)v = v$. For $w \in \mathrm{range}(A) \cap v^\perp$ (a subspace of dimension $r - 1$, since $Av \neq 0$), $Nw = \eta w$; for $w \in \mathrm{range}(A)^\perp \cap v^\perp$ (dimension $n - r - 1$, since $(I - A)v \neq 0$), $Nw = 0$. These subspaces and $\mathrm{span}(v)$ are independent and together span an $N$-invariant subspace of dimension $n - 1$, on which $N$ has eigenvalues $1$, $\eta$ ($r-1$ times), and $0$ ($n-r-1$ times). The remaining eigenvalue equals $\mathrm{tr}(N) - 1 - \eta(r - 1) = \eta r + 1 - \eta\|Av\|_2^2 - 1 - \eta(r-1) = \eta\|(I - A)v\|_2^2$.
\end{proof}

\subsection{A Sufficient Condition for Quadratic Convergence for Adaptive Rescaling}
\emph{Overview.} An exact expression for the residual after an adaptive step (Lemma~\ref{lem:A8}) shows that the step removes the component of $(I - \eta A(f_t))(\mathbf{y} - f_t)$ along $\mathbf{y}$ but not its orthogonal part. Local quadratic convergence therefore holds when the weak learners capture the residual up to second order, and not in general.

In the following theorem, we give a sufficient condition for local convergence of the training loss to zero at a quadratic rate.

\begin{theorem}[Quadratic Convergence Rate]\label{thm:quad}
Consider the adaptive rescaling variant for the squared-error loss with $\eta \in (0,1]$ and $\mathbf{y} \neq 0$. Suppose that, for some $\kappa > 0$ and all $t$, the residual $r_t = \mathbf{y} - f_t$ satisfies the local realizability condition
\[
\big\|P_{\mathbf{y}}^\perp(I - \eta A(f_t))r_t\big\|_2 \le \kappa\|r_t\|_2^2, \qquad P_{\mathbf{y}}^\perp := I - \mathbf{y}\mathbf{y}^\top/\|\mathbf{y}\|_2^2 ,
\]
and that $\|\mathbf{y} - f_0\|_2 \le \|\mathbf{y}\|_2/2$. Then $\|\mathbf{y} - f_{t+1}\|_2 \le 2\kappa\|\mathbf{y} - f_t\|_2^2$ for all $t \ge 0$, so the residual converges to zero at a quadratic rate if, in addition, $\|\mathbf{y} - f_0\|_2 < 1/(2\kappa)$.
\end{theorem}

The condition holds, for instance, with $\eta = 1$ if $\|(I - A(f))(\mathbf{y} - f)\|_2 \le \kappa\|\mathbf{y} - f\|_2^2$ near $\mathbf{y}$, i.e., if the weak learners at $f$ represent the residual up to a second-order error. It is not implied by the other assumptions of the paper: since multicalibration does not require a zero residual, the residual of the adaptive iteration can converge to a nonzero limit, e.g., when the column space of $B(f)$ does not depend on $f$ and $\mathbf{y}$ is not in the span of $f_0$ and this column space, which contains all iterates. Under the condition, the 1D line search of adaptive rescaling yields Newton-like local convergence.

\begin{proof}
We first establish an exact expression for the residual after an adaptive step.

\begin{lemma}\label{lem:A8}
If $\varphi(f_t) \neq 0$, the residuals evolve according to
\begin{equation}
r_{t+1} = P_{\varphi(f_t)}^\perp\mathbf{y}, \qquad \|r_{t+1}\|_2 = \frac{\|\mathbf{y}\|_2\,\big\|P_{\mathbf{y}}^\perp(I - \eta A(f_t))r_t\big\|_2}{\|\varphi(f_t)\|_2},
\end{equation}
where $P_v^\perp := I - vv^\top/\|v\|_2^2$.
\end{lemma}
\begin{proof}
By Eq.~(7), $f_{t+1} = \omega(f_t)\varphi(f_t)$ is the orthogonal projection of $\mathbf{y}$ onto the span of $\varphi(f_t)$, so $r_{t+1} = P_{\varphi(f_t)}^\perp\mathbf{y}$ and $\|r_{t+1}\|_2^2 = \|\mathbf{y}\|_2^2 - (\mathbf{y}^\top\varphi(f_t))^2/\|\varphi(f_t)\|_2^2$. Write $q := (I - \eta A(f_t))r_t$, so that $\varphi(f_t) = \mathbf{y} - q$. Then
\[
\|\mathbf{y}\|_2^2\|\varphi(f_t)\|_2^2 - (\mathbf{y}^\top\varphi(f_t))^2 = \|\mathbf{y}\|_2^2\|q\|_2^2 - (\mathbf{y}^\top q)^2 = \|\mathbf{y}\|_2^2\,\|P_{\mathbf{y}}^\perp q\|_2^2 ,
\]
and dividing by $\|\varphi(f_t)\|_2^2$ gives the claim.
\end{proof}

By Theorem~\ref{thm:adaptive}, $\|r_t\|_2 \le \|r_0\|_2 \le \|\mathbf{y}\|_2/2$ for all $t$. Since $\|q\|_2 = \|(I - \eta A(f_t))r_t\|_2 \le \|r_t\|_2$, we have $\|\varphi(f_t)\|_2 = \|\mathbf{y} - q\|_2 \ge \|\mathbf{y}\|_2 - \|r_t\|_2 \ge \|\mathbf{y}\|_2/2 > 0$. Lemma~\ref{lem:A8} and the local realizability condition then give
\[
\|r_{t+1}\|_2 \le \frac{\|\mathbf{y}\|_2\,\kappa\|r_t\|_2^2}{\|\mathbf{y}\|_2/2} = 2\kappa\|r_t\|_2^2 .
\]
If $\|r_0\|_2 < 1/(2\kappa)$, then $2\kappa\|r_{t+1}\|_2 \le (2\kappa\|r_t\|_2)^2$, so $2\kappa\|r_t\|_2 \le (2\kappa\|r_0\|_2)^{2^t} \to 0$.
\end{proof}

\subsection{A Hybrid Boosting Algorithm}
\emph{Overview.} This subsection presents a variant of Algorithm~\ref{alg:mcgb} that adds, in every round, a fixed multiple of the residuals of a given predictor $\hat y$ to the boosting update. On the training sample, and for the squared loss with an exact oracle, the variant is an inexact preconditioned gradient descent on the training loss. Its preconditioner $\gamma I + A(f_t)$ is bounded away from zero even though $A(f_t)$ is a projection, and its inexactness is proportional to the error of $\hat y$. A general contraction bound for such iterations (Lemma~\ref{lem:A9}, proved in Appendix A.11) then gives linear convergence of the residual to a neighbourhood of zero whose size is proportional to $\|\mathbf{y} - \hat y\|_2$ (Proposition~\ref{prop:A10}, proved in Appendix A.12). Unlike Theorem~\ref{thm:linear}, this requires no smoothness condition on the weak learners.

\textbf{The Algorithm.} We modify the update of Algorithm~\ref{alg:mcgb} to
\[
\hat f_{t+1}(x) = \hat f_t(x) + \eta\big[\delta\hat f_t(x) + \gamma(\hat y(x) - \hat f_t(x))\big],
\]
where $\hat y$ is a given predictor, $\gamma > 0$ is a fixed weight, and $\delta\hat f_t$ is the boosting update of round $t$; on the sample, $\delta\hat f_t = A(f_t)(\mathbf{y} - f_t)$ for the squared loss with an exact oracle. As before, $f_t$ denotes the vector of values of $\hat f_t$ on the sample, and we write $\hat y$ also for the vector of values of $\hat y$ on the sample.

\textbf{Interpretation.} The difference from Algorithm~\ref{alg:mcgb} is the additional term $\gamma(\hat y(x) - \hat f_t(x))$. We may regard it as adding an extra predictor with fixed weight $\gamma$, whose predictions are the approximate residuals $\hat y(x) - \hat f_t(x)$; if $\hat y$ were a perfect predictor, these would be exactly the residuals $y(x) - \hat f_t(x)$. The weak learners can only correct the part of the residual that lies in the column space of $B(f_t)$, which is why $A(f_t)$, a projection, can be badly conditioned as a preconditioner. The extra predictor acts on the whole residual and thereby removes this ill-conditioning. The method can be seen as a hybrid form of boosting, as it mixes weak learners with a strong learner.

\textbf{Preconditioned Gradient Descent Form.} On the sample, with $V(f) = \frac{1}{2}\|\mathbf{y} - f\|_2^2$, the iterates satisfy
\begin{equation}\label{eq:hybrid}
f_{t+1} = f_t - \eta\big[P(f_t)\nabla V(f_t) + \xi_t\big], \qquad P(f_t) = \gamma I + A(f_t), \quad \xi_t = \gamma(\mathbf{y} - \hat y).
\end{equation}
Indeed, since $\nabla V(f_t) = f_t - \mathbf{y}$, we have $P(f_t)\nabla V(f_t) + \xi_t = -A(f_t)(\mathbf{y} - f_t) - \gamma(\hat y - f_t)$, which is minus the bracket in the update above. Thus \eqref{eq:hybrid} is a preconditioned gradient descent on $V$, made inexact by the term $\xi_t$, which vanishes when $\hat y = \mathbf{y}$.

\textbf{Multicalibration Error.} Since $A(f_t)$ is the orthogonal projection onto the column space of $B(f_t)$, we have $B(f_t)^\top A(f_t) = B(f_t)^\top$. Multiplying the update by $B(f_t)^\top$ therefore gives $B(f_t)^\top(f_{t+1} - f_t) = \eta B(f_t)^\top(\mathbf{y} - f_t) + \eta\gamma B(f_t)^\top(\hat y - f_t)$, and hence
\[
\Eh(f_t) = \gamma\frac{1}{n}B(f_t)^\top(f_t - \hat y) + \frac{1}{\eta n}B(f_t)^\top(f_{t+1} - f_t).
\]
As for Algorithm~\ref{alg:mcgb}, the second term vanishes as the gap between successive predictions vanishes. The first term is a floor, proportional to $\gamma$ and to the distance between the current predictor and $\hat y$, that is the price of the extra predictor.

\textbf{A Contraction Bound.} The following lemma gives a general convergence result for inexact preconditioned gradient descent of the form \eqref{eq:hybrid}; we prove it in Appendix A.11.

\begin{lemma}\label{lem:A9}
Assume that $V : \R^n \to \R$ is a continuously differentiable function that is $L$-smooth and $\mu$-strongly convex, $P(f)$ is a symmetric positive definite matrix such that for some constants $0 < m \le M < \infty$, $mI \preceq P(f) \preceq MI$, and $(f - f^*)^\top P(f)\nabla V(f) \ge m(f - f^*)^\top\nabla V(f)$, where $f^*$ is the minimizer of $V$. Then, for any step size $\eta \in (0, 2m\mu/(L^2M^2))$, the iterates of \eqref{eq:hybrid} satisfy
\[
\|f^* - f_{t+1}\|_2 \le \kappa\|f^* - f_t\|_2 + \eta\|\xi_t\|_2
\]
where
\[
\kappa = \sqrt{1 - 2\mu m\eta\left(1 - \frac{L^2M^2}{2\mu m}\eta\right)} < 1.
\]
\end{lemma}

\textbf{Linear Convergence of the Hybrid Algorithm.} For the squared loss, $V(f) = \frac{1}{2}\|\mathbf{y} - f\|_2^2$ is $L$-smooth and $\mu$-strongly convex with $L = \mu = 1$, and its minimizer is $f^* = \mathbf{y}$. Since $A(f)$ is an orthogonal projection, $P(f) = \gamma I + A(f)$ satisfies $\gamma I \preceq P(f) \preceq (1 + \gamma)I$ and $(f - \mathbf{y})^\top P(f)(f - \mathbf{y}) \ge \gamma\|f - \mathbf{y}\|_2^2$. Lemma~\ref{lem:A9} therefore applies with $m = \gamma$ and $M = 1 + \gamma$, and with $\|\xi_t\|_2 = \gamma\|\mathbf{y} - \hat y\|_2$ it gives the following result (the details are in Appendix A.12).

\begin{proposition}\label{prop:A10}
Assume that the shrinkage parameter satisfies $\eta \in (0, 2\gamma/(1 + \gamma)^2)$. Then, for all $t \ge 0$,
\[
\|\mathbf{y} - f_{t+1}\|_2 \le \kappa\|\mathbf{y} - f_t\|_2 + \eta\gamma\|\mathbf{y} - \hat y\|_2
\]
where
\[
\kappa = \sqrt{1 - 2\gamma\eta\left(1 - \frac{(1 + \gamma)^2}{2\gamma}\eta\right)} < 1.
\]
\end{proposition}

\textbf{Consequences and Choice of the Step Size.} Unrolling the recursion of Proposition~\ref{prop:A10} and summing the geometric series,
\[
\|\mathbf{y} - f_t\|_2 \le \kappa^t\|\mathbf{y} - f_0\|_2 + \frac{\eta\gamma}{1 - \kappa}\|\mathbf{y} - \hat y\|_2 .
\]
The residual thus converges linearly, at a rate $\kappa$ that does not depend on $A(f_t)$, to a neighbourhood of zero whose size is proportional to the error $\|\mathbf{y} - \hat y\|_2$ of the strong learner. If the step size is set to $\eta = c\gamma$ for some constant $0 < c < 1$, the condition of Proposition~\ref{prop:A10} becomes $\gamma < \sqrt{2/c} - 1$. For small $\gamma$, the contraction factor is then $\kappa \approx 1 - c(1 - c/2)\gamma^2$, and the size of the neighbourhood is governed by $\eta\gamma/(1 - \kappa) \approx 2/(2 - c)$. A smaller weight $\gamma$ therefore slows down the convergence, but does not enlarge the neighbourhood.

\subsection{Proof of Lemma A.9}
\emph{Overview.} The proof is a standard contraction argument for inexact preconditioned gradient descent. We first bound the error after an \emph{exact} step $f_t - \eta P(f_t)\nabla V(f_t)$, that is, with the inexactness $\xi_t$ removed. Strong convexity of $V$, together with the lower bound on $P$, gives a descent term proportional to $\|f_t - f^*\|_2^2$, and smoothness of $V$, together with the upper bound on $P$, bounds the quadratic term. The inexactness then enters additively, by the triangle inequality.

\textbf{Consequences of Smoothness and Strong Convexity.} By the $\mu$-strong convexity of $V$, for every $f, f'$,
\[
(f - f')^\top(\nabla V(f) - \nabla V(f')) \ge \mu\|f - f'\|_2^2,
\]
and, by the $L$-smoothness of $V$,
\[
\|\nabla V(f) - \nabla V(f')\|_2 \le L\|f - f'\|_2.
\]
Taking $f' = f^*$ and using $\nabla V(f^*) = 0$, we obtain, for every $f$,
\begin{equation}\label{eq:A9sc}
(f - f^*)^\top\nabla V(f) \ge \mu\|f - f^*\|_2^2
\end{equation}
and
\begin{equation}\label{eq:A9sm}
\|\nabla V(f)\|_2 \le L\|f - f^*\|_2.
\end{equation}

\textbf{The Exact Step.} Let $\delta_t := f_t - f^* - \eta P(f_t)\nabla V(f_t)$ be the error after an exact step. Expanding the square,
\[
\|\delta_t\|_2^2 = \|f_t - f^*\|_2^2 - 2\eta(f_t - f^*)^\top P(f_t)\nabla V(f_t) + \eta^2\|P(f_t)\nabla V(f_t)\|_2^2 .
\]
We bound the two terms involving $P$ in turn. For the cross term, the assumption $(f_t - f^*)^\top P(f_t)\nabla V(f_t) \ge m(f_t - f^*)^\top\nabla V(f_t)$, combined with \eqref{eq:A9sc}, gives
\[
(f_t - f^*)^\top P(f_t)\nabla V(f_t) \ge m\mu\|f_t - f^*\|_2^2 .
\]
For the quadratic term, $P(f_t)$ is symmetric with $P(f_t) \preceq MI$, so $\|P(f_t)\|_2 \le M$, which combined with \eqref{eq:A9sm} gives
\[
\|P(f_t)\nabla V(f_t)\|_2 \le ML\|f_t - f^*\|_2 .
\]
Substituting both bounds,
\[
\|\delta_t\|_2^2 \le (1 - 2\eta m\mu + \eta^2M^2L^2)\|f_t - f^*\|_2^2 = \kappa^2\|f_t - f^*\|_2^2 ,
\]
with $\kappa^2 = 1 - 2\eta m\mu + \eta^2M^2L^2 = 1 - 2\mu m\eta\big(1 - L^2M^2\eta/(2\mu m)\big)$. The step-size condition $\eta < 2m\mu/(L^2M^2)$ ensures that $\kappa^2 < 1$. Moreover, $\kappa^2 \ge (1 - \eta m\mu)^2 \ge 0$ because $m\mu \le ML$, so $\kappa$ is well defined.

\textbf{Adding the Inexactness.} By \eqref{eq:hybrid}, $f_{t+1} - f^* = \delta_t - \eta\xi_t$. By the triangle inequality,
\[
\|f_{t+1} - f^*\|_2 \le \|\delta_t\|_2 + \eta\|\xi_t\|_2 \le \kappa\|f_t - f^*\|_2 + \eta\|\xi_t\|_2 ,
\]
which proves the lemma. \qed

\subsection{Proof of Proposition A.10}
\emph{Overview.} Proposition~\ref{prop:A10} is the special case of Lemma~\ref{lem:A9} for the squared loss and the preconditioner $P(f) = \gamma I + A(f)$; it remains to verify the assumptions of the lemma and to identify its constants.

The training loss $V(f) = \frac{1}{2}\|f - \mathbf{y}\|_2^2$ is $L$-smooth and $\mu$-strongly convex with $\mu = L = 1$, and its minimizer is $f^* = \mathbf{y}$. Since $A(f)$ is an orthogonal projection, $P(f) = \gamma I + A(f)$ is symmetric with $\gamma I \preceq P(f) \preceq (1 + \gamma)I$, so $m = \gamma$ and $M = 1 + \gamma$. Since $\nabla V(f) = f - f^*$, the remaining assumption of Lemma~\ref{lem:A9} reads
\[
(f - f^*)^\top P(f)(f - f^*) \ge m\|f - f^*\|_2^2 ,
\]
which holds because $mI \preceq P(f)$. With these constants, the step-size condition $\eta < 2m\mu/(L^2M^2)$ of Lemma~\ref{lem:A9} becomes $\eta < 2\gamma/(1 + \gamma)^2$, its contraction factor $\kappa$ becomes the one stated in Proposition~\ref{prop:A10}, and $\|\xi_t\|_2 = \gamma\|\mathbf{y} - \hat y\|_2$ gives the additive term $\eta\gamma\|\mathbf{y} - \hat y\|_2$. \qed

\subsection{Existence of the Exact Oracle and Assumption (A2) for the Log-Loss}
\emph{Overview.} For the log-loss, a minimizer of the round's subproblem need not exist when the weak learners separate the labels. Lemma~\ref{lem:logloss-exist} below gives the standard non-separation condition under which it exists and is unique; (A2) then holds for any fixed number of rounds.

\textbf{Smoothness and Strong Convexity of the Log-Loss.} For binary classification with the log-loss $L(y, f) = -y \log(\phi(f)) - (1 - y) \log(1 - \phi(f))$, where $\phi(f) := 1/(1 + e^{-f})$ and labels take values in $\{0, 1\}$, the training loss $V(f)$ is globally smooth with $L = 1/4$, but its strong convexity vanishes as $\|f\|_\infty \to \infty$. Indeed,
\[
\nabla^2 V(f) = \mathrm{diag}\big(\phi(f(x_1))(1 - \phi(f(x_1))), \dots , \phi(f(x_n))(1 - \phi(f(x_n)))\big).
\]
Hence, the smallest eigenvalue of the Hessian is $\lambda_{\min}(\nabla^2 V(f)) = \phi(\|f\|_\infty)(1 - \phi(\|f\|_\infty))$ which goes to 0 as $\|f\|_\infty \to \infty$. Assumption (A2) therefore holds only if $\max_{0\le\tau\le1} \|f_t + \tau h_t\|_\infty$ is bounded for all $t = 0, 1, \dots , T - 1$; Lemma~\ref{lem:logloss-exist} gives a condition under which this is the case.

\begin{lemma}\label{lem:logloss-exist}
Consider the log-loss and a round $t$, and write $B = B(f_t)$, $s_i = 2y_i - 1 \in \{-1, +1\}$. Suppose there is no $v \in \R^p\setminus\{0\}$ with $s_i(Bv)_i \ge 0$ for all $i$. Then $B$ has full column rank, $\theta \mapsto V(f_t + B\theta)$ is strictly convex and coercive, and the exact oracle's minimizer $\theta_t$ exists and is unique. If this holds for every $t < T$, then (A2) holds with
\[
m = \min_{0\le t<T}\ \min_{0\le\tau\le1}\ \min_{i}\ \phi'\big(f_t(x_i) + \tau h_t(x_i)\big) > 0, \qquad \phi' = \phi(1 - \phi).
\]
Conversely, if some $v \neq 0$ satisfies $s_i(Bv)_i > 0$ for all $i$, the infimum of $V(f_t + B\theta)$ is not attained.
\end{lemma}
\begin{proof}
If $Bv = 0$ for some $v \neq 0$, then $s_i(Bv)_i = 0$ for all $i$, contradicting the assumption, so $B$ has full column rank. The Hessian in $\theta$ is $B^\top WB$ with $W = \mathrm{diag}(\phi'(\cdot)) \succ 0$, hence positive definite, so the function is strictly convex. For coercivity, note $L(1, z) = \log(1 + e^{-z}) \ge \max(0, -z)$ and $L(0, z) = \log(1 + e^{z}) \ge \max(0, z)$, i.e., $L(y_i, z) \ge \max(0, -s_iz)$. The function $c(v) := \max_i(-s_i(Bv)_i)$ is continuous and, by assumption, positive on the unit sphere, so $c(v) \ge c_0 > 0$ there. Hence, for $\theta \neq 0$,
\[
V(f_t + B\theta) \ge \max_i\big(-s_if_t(x_i) - s_i(B\theta)_i\big) \ge c_0\|\theta\|_2 - \max_i|f_t(x_i)| \to \infty \quad (\|\theta\|_2 \to \infty),
\]
so a minimizer exists, and it is unique by strict convexity. Then $h_t = B\theta_t$ is finite, the segment $\{f_t + \tau h_t : 0 \le \tau \le 1\}$ is bounded, and $\nabla^2V = \mathrm{diag}(\phi'(\cdot)) \succeq m_tI$ on it with $m_t > 0$; taking the minimum over $t < T$ gives (A2). For the converse, every term $L(y_i, f_t(x_i) + (B\theta)_i + \alpha(Bv)_i)$ is strictly decreasing in $\alpha$, so $V(f_t + B(\theta + \alpha v)) < V(f_t + B\theta)$ for every $\theta$ and $\alpha > 0$, and no minimizer exists.
\end{proof}

\section{The Boosting Oracle and Practical Gradient Boosting}\label{app:oracle}
Within a round, the features $(x, f_t(x))$ are fixed, so the inner loop is a gradient boosting machine over a fixed dictionary of weak learners, and convergence results for such machines bound the oracle error $\epsilon_t$ of assumption (A3). By \citet{lu2020}, in expectation over the random selection of weak learners, $\epsilon_t$ decays geometrically in the number $M_t$ of inner iterations for smooth, strongly convex losses such as the squared loss, at a rate set by the minimal cosine angle of the weak learners (their Theorem 4.1). For smooth losses with bounded level sets, it decays as $O(1/M_t)$ (their Theorem 4.2; the log-loss needs a small ridge term). If the constants in these bounds are uniform over rounds, then $\sum_t\epsilon_t < \infty$ almost surely, provided $M_t \ge c\log t$ for a large enough constant $c$ in the first case, and $\sum_t 1/M_t < \infty$ (for instance, $M_t = t^2$) in the second. A constant $M_t$ gives the noise floor discussed after Theorem~\ref{thm:main}. Greedy tree fitting only approximates the exact weak-learner selection that these results assume.

\section{Proofs for Section 5}\label{app:gen}
\emph{Overview.} The proof of Theorem~\ref{thm:pop} has three parts. (i) \emph{Empirical term:} Lemma~\ref{lem:A1} (first-order optimality plus smoothness) and Theorem~\ref{thm:main}. (ii) \emph{Covering numbers:} we bound the $\ell_\infty$ covering numbers of the class of $T$-round predictors by chaining one-round covers (Lemma~\ref{lem:coverF}), and of the multicalibration-error class by a ``cover of covers'' (Lemma~\ref{lem:coverE}); each is a standard covering argument, adapted to the fact that the weak learner at round $t$ is evaluated at the previous prediction. (iii) \emph{Uniform convergence and $\ell_2$-duality:} the Maurer--Pontil bound gives a uniform bound over the class, and $\ell_2$-duality turns it into a bound on the norm of the error vector. Throughout, $\phi$ is the sigmoid, so $|y - \phi(f)| \le 1$ and $\phi$ is $1/4$-Lipschitz; we write $L_\phi = 1/4$ and $B_r = 1$ to indicate where these constants enter.

\subsection{Covering Numbers}
\emph{Overview.} This subsection bounds the covering number of the class of error functions $e_{\theta,f}$ whose population-minus-empirical means appear in the proof of Theorem~\ref{thm:pop}. The argument proceeds in three steps, one lemma each. Lemma~\ref{lem:coverS} covers the class of weak-learner combinations with bounded coefficients; this class is finite-dimensional, and its covering number grows like $(1/\varepsilon)^p$. Lemma~\ref{lem:coverF} covers the class of predictors that the algorithm can reach in $T$ rounds, by propagating covers of the weak-learner combinations through the rounds; because the prediction is fed back as a feature, an approximation error made in an early round is amplified by the per-round Lipschitz constant $\Gamma$ in every later round. Lemma~\ref{lem:coverE} combines the two to cover the class of error functions. The resulting bound enters the uniform convergence bound of Appendix C.2 through the quantity $\Lambda$.

\textbf{Covering Numbers.} For a class $\mathcal{G}$ of functions on a domain $\mathcal{Z}$ and a sample $z = (z_1, \dots, z_n) \in \mathcal{Z}^n$, let $\mathcal{G}(z) = \{(g(z_1),\dots,g(z_n)) : g \in \mathcal{G}\} \subset \R^n$, and define
\[
N_\infty(\varepsilon, \mathcal{G}, n) := \sup_{z \in \mathcal{Z}^n}N\big(\varepsilon, \mathcal{G}(z), \|\cdot\|_\infty\big),
\]
where $N(\varepsilon, A, \|\cdot\|_\infty)$ is the smallest number of $\varepsilon$-balls in $\|\cdot\|_\infty$, \emph{centred at points of $A$}, that cover $A$ (the convention of \citealt{maurer2009}). All covers constructed below have this property: their centres are value vectors of members of the class.

\textbf{Weak-Learner Combinations.} For $R > 0$ let
\[
\mathcal{S}_R := \{(x,u) \mapsto \theta^\top b(x,u) : \theta \in \R^p, \|\theta\|_2 \le R\}, \qquad \mathcal{H}_1 := \mathcal{S}_1 .
\]
Under (G2), every $h \in \mathcal{S}_R$ satisfies $|h(x,u)| \le RB_2$ and $|h(x,u) - h(x,u')| \le RL_b|u - u'|$.

\begin{lemma}[Covering the weak-learner combinations]\label{lem:coverS}
Under (G2), for every $\varepsilon > 0$ and every $n$,
\[
N_\infty(\varepsilon, \mathcal{S}_R, n) \le \Big(1 + \frac{2RB_2}{\varepsilon}\Big)^p .
\]
\end{lemma}
\begin{proof}
For $\|\theta\|_2, \|\tilde\theta\|_2 \le R$ and any $(x,u)$, $|(\theta - \tilde\theta)^\top b(x,u)| \le B_2\|\theta - \tilde\theta\|_2$, uniformly over points. Hence an $(\varepsilon/B_2)$-cover of the $\ell_2$-ball of radius $R$ induces an $\varepsilon$-cover of $\mathcal{S}_R$ on any sample. A volumetric argument bounds the number of $(\varepsilon/B_2)$-balls, centred in the ball, needed to cover it by $(1 + 2RB_2/\varepsilon)^p$ \citep{vershynin2018}; the centres correspond to members of $\mathcal{S}_R$.
\end{proof}

\textbf{Predictors Reachable in $T$ Rounds.} Define the classes of predictors reachable in $t$ rounds recursively by
\[
\mathcal{F}_0 := \{f_0\}, \qquad \mathcal{F}_{t+1} := \big\{f + \eta\, h(\cdot, f(\cdot)) : f \in \mathcal{F}_t,\ h \in \mathcal{S}_{B_\theta}\big\} \quad (t \ge 0),
\]
where $\mathcal{F}_0$ is a single predictor by (G1). Under (G3) the iterates satisfy $f_t \in \mathcal{F}_t$, and $\mathcal{F}_t \subseteq \mathcal{F}_{t+1}$ since $0 \in \mathcal{S}_{B_\theta}$.

\begin{lemma}[Covering the $T$-round predictors]\label{lem:coverF}
Under (G1)--(G2), with $\Gamma = 1 + \eta B_\theta L_b$,
\[
N_\infty(\varepsilon, \mathcal{F}_T, n) \le \prod_{s=0}^{T-1}N_\infty\Big(\frac{\varepsilon}{T\eta\Gamma^{T-1-s}}, \mathcal{S}_{B_\theta}, n\Big).
\]
\end{lemma}
\begin{proof}
Fix a sample $x_1,\dots,x_n$ and identify functions with their value vectors in $\R^n$. For $h \in \mathcal{S}_{B_\theta}$ let $G_h(f) := f + \eta\,(h(x_i, f_i))_{i\le n}$. Since $h$ is $B_\theta L_b$-Lipschitz in its second argument, $\|G_h(f) - G_h(f')\|_\infty \le \Gamma\|f - f'\|_\infty$.

We construct, by induction on $t$, sets $C_t \subset \R^n$ with $|C_t| \le \prod_{s<t}N_s$, $N_s := N_\infty(\beta_s, \mathcal{S}_{B_\theta}, n)$, such that every $f_t \in \mathcal{F}_t$ is within $\varepsilon_t := \eta\sum_{s<t}\Gamma^{t-1-s}\beta_s$ of some element of $C_t$ in $\|\cdot\|_\infty$. For $t=0$, $C_0 = \{f_0\}$ and $\varepsilon_0 = 0$. For the step, for \emph{each} $\tilde f \in C_t$ take a $\beta_t$-cover $\hat C(\tilde f)$ of $\mathcal{S}_{B_\theta}$ evaluated at the $n$ points $(x_i, \tilde f_i)$; it exists with at most $N_t$ elements because $N_\infty$ is a supremum over all sample points. Set $C_{t+1} = \{G_{\tilde h}(\tilde f) : \tilde f \in C_t, \tilde h \in \hat C(\tilde f)\}$. Given $f_{t+1} = G_h(f_t)$ with $f_t \in \mathcal{F}_t$, pick $\tilde f \in C_t$ with $\|f_t - \tilde f\|_\infty \le \varepsilon_t$ and $\tilde h \in \hat C(\tilde f)$ within $\beta_t$ of $h$ at the points $(x_i, \tilde f_i)$. Then
\[
\|f_{t+1} - G_{\tilde h}(\tilde f)\|_\infty \le \|G_h(f_t) - G_h(\tilde f)\|_\infty + \|G_h(\tilde f) - G_{\tilde h}(\tilde f)\|_\infty \le \Gamma\varepsilon_t + \eta\beta_t = \varepsilon_{t+1}.
\]
By induction, every element of $C_t$ is the value vector of a member of $\mathcal{F}_t$ (the centres $\tilde h$ are members of $\mathcal{S}_{B_\theta}$). Choosing $\beta_s = \varepsilon/(T\eta\Gamma^{T-1-s})$ gives $\varepsilon_T = \sum_{s<T}\varepsilon/T = \varepsilon$.
\end{proof}

\textbf{The Error Class.} The error functions and the class they form are
\[
e_{\theta,f}(x,y) := \theta^\top b(x, f(x))\,\big(y - \phi(f(x))\big),
\]
\[
\mathcal{E}_1 := \big\{e_{\theta,f} : \theta \in \R^p,\ \|\theta\|_2 \le 1,\ f \in \mathcal{F}_T\big\}.
\]
Since $|\theta^\top b(x,u)| \le B_2$ for $\|\theta\|_2 \le 1$ and $|y - \phi(f(x))| \le B_r = 1$, every $e \in \mathcal{E}_1$ satisfies $|e| \le B_2$.

\begin{lemma}[Covering the error class]\label{lem:coverE}
Under (G2), for every $\varepsilon > 0$,
\[
N_\infty(\varepsilon, \mathcal{E}_1, n) \le N_\infty\Big(\frac{\varepsilon}{2(B_2L_\phi + B_rL_b)}, \mathcal{F}_T, n\Big)\, N_\infty\Big(\frac{\varepsilon}{2B_r}, \mathcal{H}_1, n\Big).
\]
\end{lemma}
\begin{proof}
Fix a sample $z_i = (x_i, y_i)$. Take an $\alpha$-cover $\mathcal{F}_\alpha$ of $\mathcal{F}_T$ on $x_1,\dots,x_n$, and for each $\tilde f \in \mathcal{F}_\alpha$ a $\beta$-cover $\mathcal{C}(\tilde f)$ of $\mathcal{H}_1$ at the points $(x_i, \tilde f(x_i))$ (a ``cover of covers''; it has at most $N_\infty(\beta, \mathcal{H}_1, n)$ elements). For $(\theta, f)$ choose $\tilde f$ with $\|f - \tilde f\|_\infty \le \alpha$ and then $\tilde h \in \mathcal{C}(\tilde f)$ with $|h_\theta(x_i, \tilde f_i) - \tilde h(x_i, \tilde f_i)| \le \beta$. For each $i$, adding and subtracting $h_\theta(x_i, f_i)(y_i - \phi(\tilde f_i))$,
\begin{align*}
|e_{\theta,f}(z_i) - \tilde h(x_i,\tilde f_i)(y_i - \phi(\tilde f_i))| &\le |h_\theta(x_i,f_i)|\,|\phi(f_i) - \phi(\tilde f_i)|\\
&\quad + |y_i - \phi(\tilde f_i)|\,|h_\theta(x_i,f_i) - \tilde h(x_i,\tilde f_i)|\\
&\le B_2L_\phi\alpha + B_r(L_b\alpha + \beta),
\end{align*}
using that $h_\theta$ is $L_b$-Lipschitz in its second argument. This is at most $\varepsilon$ for $\alpha = \varepsilon/(2(B_2L_\phi + B_rL_b))$ and $\beta = \varepsilon/(2B_r)$, and the cover, whose elements are value vectors of members $e_{\tilde\theta,\tilde f}$ of $\mathcal{E}_1$, has at most $N_\infty(\alpha, \mathcal{F}_T, n)N_\infty(\beta, \mathcal{H}_1, n)$ elements.
\end{proof}

\subsection{Uniform Convergence}
\emph{Overview.} This subsection turns the covering bounds of Appendix C.1 into a bound on the difference between the population and the empirical mean of the error functions $e_{\theta,f}$ that holds \emph{simultaneously} for all $e_{\theta,f} \in \mathcal{E}_1$. Uniformity is essential here: in the proof of Theorem~\ref{thm:pop}, the bound is applied at the pair $(\theta^*, f_{t^*})$, which depends on the training sample, and a bound for a fixed pair would not suffice. We use the empirical Bernstein bound of \citet{maurer2009} (Theorem~\ref{thm:MP} below), which holds uniformly over a class of $[0,1]$-valued functions. Its leading term is proportional to the sample standard deviation of the function, rather than to its range, and the complexity of the class enters only through a covering number at scale $1/n$ on a sample of size $2n$. Lemma~\ref{lem:unif} then applies this bound to the error class in two steps: it rescales the error functions, which take values in $[-B_2, B_2]$, to the interval $[0,1]$, and it bounds the resulting covering number by the lemmas of Appendix C.1, which yields the quantity $\Lambda$.

For a vector $Z = (Z_1, \dots, Z_n)$, we denote the sample variance by
\[
V_n(Z) := \frac{1}{n(n-1)}\sum_{i<j}(Z_i - Z_j)^2 .
\]

\begin{theorem}[Theorem 6 of \citealt{maurer2009}]\label{thm:MP}
Let $\mathcal{F}$ be a class of functions $g : \mathcal{Z} \to [0,1]$, $\delta \in (0,1)$, $n \ge 16$, and $\mathcal{M}(n) = 10\,N_\infty(1/n, \mathcal{F}, 2n)$. With probability at least $1-\delta$, for all $g \in \mathcal{F}$,
\[
\E_D[g] - \E_n[g] \le \sqrt{\frac{18V_n(g(Z))\log(\mathcal{M}(n)/\delta)}{n}} + \frac{15\log(\mathcal{M}(n)/\delta)}{n-1}.
\]
\end{theorem}

\begin{lemma}[Uniform bound for the error class]\label{lem:unif}
Under (G1)--(G2), for $n \ge 16$ and $\delta \in (0,1)$, with probability at least $1-\delta$, simultaneously for all $\|\theta\|_2 \le 1$ and $f \in \mathcal{F}_T$,
\[
\E_D[e_{\theta,f}] - \E_n[e_{\theta,f}] \le \sqrt{\frac{18V_n(e_{\theta,f}(Z))\Lambda}{n}} + \frac{30B_2\Lambda}{n-1}.
\]
\end{lemma}
\begin{proof}
Apply Theorem~\ref{thm:MP} to $\tilde{\mathcal{E}} = \{(e + B_2)/(2B_2) : e \in \mathcal{E}_1\}$, which is $[0,1]$-valued. The affine map scales $\|\cdot\|_\infty$-distances by $1/(2B_2)$, so $N_\infty(1/n, \tilde{\mathcal{E}}, 2n) = N_\infty(2B_2/n, \mathcal{E}_1, 2n)$, and $V_n$ scales by $1/(2B_2)^2$. Multiplying the resulting inequality by $2B_2$ gives the stated form with $\log(\mathcal{M}(n)/\delta)$ in place of $\Lambda$. It remains to show $\log(\mathcal{M}(n)/\delta) \le \Lambda$. By Lemma~\ref{lem:coverE} with $\varepsilon = 2B_2/n$, $\beta = \varepsilon/(2B_r) = B_2/n$ and Lemma~\ref{lem:coverS}, $\log N_\infty(\beta, \mathcal{H}_1, 2n) \le p\log(1 + 2n)$. With $\alpha = \varepsilon/(2(B_2/4 + L_b)) = B_2/(n(B_2/4 + L_b))$, Lemmas~\ref{lem:coverF} and~\ref{lem:coverS} give
\[
\log N_\infty(\alpha, \mathcal{F}_T, 2n) \le p\sum_{s=0}^{T-1}\log\Big(1 + \frac{2B_\theta B_2T\eta\Gamma^{T-1-s}}{\alpha}\Big) = p\sum_{s=0}^{T-1}\log\big(1 + KnT\Gamma^{T-1-s}\big),
\]
and since $1 + a\Gamma^k \le \Gamma^k(1+a)$ for $\Gamma \ge 1$, the right-hand side is at most $pT\log(1 + KnT) + \frac p2T(T-1)\log\Gamma$. Adding $\log 10$ and $\log(1/\delta)$ gives $\Lambda$.
\end{proof}

\subsection{Proof of Theorem~\ref{thm:pop}}
\emph{Overview.} The proof splits the population multicalibration error into the empirical error, which is bounded by the optimization analysis of Section 4, and the generalization gap, which is bounded by the uniform convergence bound of Lemma~\ref{lem:unif} combined with $\ell_2$-duality.

By the triangle inequality, $\|\mathcal{E}(f_{t^*})\|_2 \le \|\Eh(f_{t^*})\|_2 + \|\mathcal{E}(f_{t^*}) - \Eh(f_{t^*})\|_2$.

\textbf{Empirical Term.} By Lemma~\ref{lem:A1} with $\epsilon_t = 0$, $\|B(f_{t^*})\|_2 \le \sqrt n\bar M$, and Theorem~\ref{thm:main} (exact oracle) in the form $\|f_{t^*+1} - f_{t^*}\|_2 \le \sqrt{4\eta V(f_0)/(mT)} = 2\sqrt{\eta n\bar V(f_0)/(mT)}$,
\[
\|\Eh(f_{t^*})\|_2 \le \frac{\sqrt n\bar M}{n}\cdot\frac{L}{\eta}\cdot2\sqrt{\frac{\eta n\bar V(f_0)}{mT}} = 2L\bar M\sqrt{\frac{\bar V(f_0)}{\eta m}}\frac{1}{\sqrt T}.
\]

\textbf{Generalization Term.} Note that $\theta^\top\mathcal{E}(f) = \E_D[e_{\theta,f}]$ and $\theta^\top\Eh(f) = \E_n[e_{\theta,f}]$ by linearity. If $\mathcal{E}(f_{t^*}) = \Eh(f_{t^*})$ there is nothing to prove; otherwise let $\theta^* = (\mathcal{E}(f_{t^*}) - \Eh(f_{t^*}))/\|\mathcal{E}(f_{t^*}) - \Eh(f_{t^*})\|_2$, so that $\|\mathcal{E}(f_{t^*}) - \Eh(f_{t^*})\|_2 = \E_D[e_{\theta^*,f_{t^*}}] - \E_n[e_{\theta^*,f_{t^*}}]$ (this is the $\ell_2$-duality $\|v\|_2 = \sup_{\|\theta\|_2\le1}\theta^\top v$). Although $\theta^*$ and $f_{t^*}$ depend on the sample, (G3) gives $f_{t^*} \in \mathcal{F}_T$, and the event of Lemma~\ref{lem:unif} is uniform over $\|\theta\|_2 \le 1$ and $f \in \mathcal{F}_T$, so it applies at $(\theta^*, f_{t^*})$. Finally, since $|y - \phi| \le 1$,
\[
V_n(e_{\theta^*,f_{t^*}}(Z)) \le \frac{1}{n-1}\sum_{i=1}^ne_{\theta^*,f_{t^*}}(z_i)^2 \le \frac{\|B(f_{t^*})\theta^*\|_2^2}{n-1} \le \frac{n\bar M^2}{n-1},
\]
so $\sqrt{18V_n\Lambda/n} \le \bar M\sqrt{18\Lambda/(n-1)}$. Adding the two terms proves the theorem. \qed

\subsection{Proof of Corollary~\ref{cor:T}}
\emph{Overview.} The corollary follows by choosing $T$ to balance the optimization term of Theorem~\ref{thm:pop}, which decreases like $T^{-1/2}$, against the leading part of the generalization term, which grows like $T\sqrt{p\log\Gamma/n}$.

Let $s := p\log\Gamma/n \le 1$ and $T \asymp s^{-1/3}$ (so $T \ge 1$). The optimization term is of order $T^{-1/2} \asymp s^{1/6}$. For $\Gamma > 1$ constant, $\Lambda \lesssim pT^2\log\Gamma + pT\log(KnT) + p\log n + \log(1/\delta)$, so $\sqrt{\Lambda/n}$ is at most of order $T\sqrt{p\log\Gamma/n} = s^{1/6}$ plus lower-order terms of order $(p/n)^{1/3}$ and $\sqrt{p\log n/n}$ (up to logarithms). The last term is of order $\Lambda/n \lesssim pT^2\log\Gamma/n = s^{1/3}$ plus lower-order terms. Since $s \le 1$ and $p \le n$, all terms are $\tilde O(s^{1/6})$, with constants depending on $\Gamma$. \qed

\subsection{Scope and Extensions of Theorem~\ref{thm:pop}}\label{app:gen-ext}
\emph{Overview.} This subsection discusses which assumptions of Theorem~\ref{thm:pop} can be relaxed, and at what cost. An inexact oracle and an initial predictor chosen from a class require only small changes to the proof; the squared loss requires bounded labels and changes the trade-off in $T$; rescaling, infinite weak-learner classes, and weak learners that violate (G2) are not covered and would require new arguments.

\textbf{Inexact Oracle.} With $\epsilon_t > 0$, Lemma~\ref{lem:A1} and Theorem~\ref{thm:main} add to the optimization term a noise floor of order $(\bar M/\sqrt n)\big(2L\sqrt{\bar\epsilon/m} + \sqrt{2L\bar\epsilon}\big)$, where $\bar\epsilon = \max_t\epsilon_t$; the generalization term is unchanged.

\textbf{Initial Predictor Chosen from a Class.} If $f_0$ is selected from a class $\mathcal{F}_0$ using the sample, start the induction in the proof of Lemma~\ref{lem:coverF} with an $\varepsilon_0$-cover of $\mathcal{F}_0$ instead of $C_0 = \{f_0\}$. The recursion becomes $\varepsilon_T = \Gamma^T\varepsilon_0 + \eta\sum_{s<T}\Gamma^{T-1-s}\beta_s$; with $\varepsilon_0 = \varepsilon/(2\Gamma^T)$ and $\beta_s = \varepsilon/(2T\eta\Gamma^{T-1-s})$,
\[
N_\infty(\varepsilon, \mathcal{F}_T, n) \le N_\infty\Big(\frac{\varepsilon}{2\Gamma^T}, \mathcal{F}_0, n\Big)\prod_{s=0}^{T-1}N_\infty\Big(\frac{\varepsilon}{2T\eta\Gamma^{T-1-s}}, \mathcal{S}_{B_\theta}, n\Big),
\]
so $\Lambda$ gains the term $\log N_\infty(\alpha/(2\Gamma^T), \mathcal{F}_0, 2n)$ (with $\alpha$ as in the proof of Lemma~\ref{lem:unif}) and $K$ doubles.

\textbf{Squared Loss.} Lemma~\ref{lem:A1} holds with $L = 1$. If $|y| \le B_y$, every $f \in \mathcal{F}_T$ satisfies $\|f\|_\infty \le \|f_0\|_\infty + T\eta B_\theta B_2$, so Lemmas~\ref{lem:coverS}--\ref{lem:unif} apply with $L_\phi = 1$ and $B_r = B_y + \|f_0\|_\infty + T\eta B_\theta B_2$, with $B_2B_r$ in place of $B_2$ in the rescaling of Lemma~\ref{lem:unif}, and with an extra factor $B_r^2$ in the variance bound of Appendix C.3. Because $B_r$ grows with $T$, the trade-off of Corollary~\ref{cor:T} changes; we have not worked out the resulting rate.

\textbf{Rescaling.} The analysis is for unit weights $w_t = 1$; relaxed and adaptive rescaling are not covered.

\textbf{Infinite Weak-Learner Classes.} If the weak learners form a ball of a reproducing kernel Hilbert space whose sup-norm entropy satisfies $\log N_\infty(\varepsilon) = O(\varepsilon^{-a})$ for some $a > 0$, this bound can replace Lemma~\ref{lem:coverS}. In Lemma~\ref{lem:coverF}, however, the per-round radii $\varepsilon/(T\eta\Gamma^{T-1-s})$ then enter polynomially rather than logarithmically, and the resulting bound on $\log N_\infty(\varepsilon, \mathcal{F}_T, n)$ grows like $\Gamma^{aT}$, exponentially in $T$ rather than quadratically. Handling such classes requires an argument that does not compound covers across rounds; we leave this open.

\textbf{Removing the Lipschitz Condition (G2).} Weak learners that depend on the prediction only through a fixed partition of its range, i.e., piecewise constant in $u$ (as trees that split the score at fixed thresholds), violate (G2). For such classes, a combinatorial argument that bounds the pseudo-dimension of the $T$-round class, as in analyses of level-set boosting \citep{globus2023}, could replace the covering recursion and appears to give a complexity that grows only linearly in $T$ for the squared loss. We have not verified this extension in full, in particular for the log-loss, where the error functions are not polynomial in the parameters, and leave it to future work.

\section{Supplementary Material for the Experiments}\label{app:exp}
This appendix describes the datasets (Appendix D.1) and the computation of the multicalibration error (Appendix D.2), and gives the full discussion of the experiments summarized in Section 6 (Appendices D.3--D.5).

\subsection{Information About Datasets}
\begin{table}[h]
\caption{Statistics of the datasets including the count of raw features, the total dimensionality ($d$) after pre-processing, the dataset size ($n$).}
\label{tab:data}
\centering
\begin{tabular}{lrrr}
\toprule
Dataset & \# Raw Features & \# Features & \# Data Points\\
\midrule
California Housing & 8 & 8 & 20640\\
Diabetes & 10 & 10 & 442\\
Adult & 13 & 101 & 32561\\
German Credit & 10 & 48 & 1000\\
Communities and Crime & 122 & 122 & 1994\\
\bottomrule
\end{tabular}
\end{table}

\textbf{California Housing.} This dataset, from the StatLib repository, contains information about California districts and is used to predict the median house value. It includes data on housing characteristics, population, income, and geographic location, making it useful for regression tasks and geospatial analysis.

\textbf{Diabetes.} Sourced from the UCI Machine Learning Repository, the diabetes dataset focuses on predicting disease progression one year after baseline measurements. It includes personal and medical indicators such as age, BMI, blood pressure, and blood serum levels, and is commonly used for regression modeling in healthcare studies.

\textbf{Adult (Census Income).} Also from the UCI Machine Learning Repository, the Adult dataset aims to classify whether an individual earns more than \$50K per year based on census data. It incorporates demographic and employment-related attributes, making it suitable for binary classification and regression tasks and studies on fairness and social analytics. In our study, the prediction target is defined to be the sum \texttt{education-num} (a numeric encoding of the categorical education feature) and \texttt{hours-per-week} (the number of hours an individual works per week). Both \texttt{education-num} and \texttt{hours-per-week} are removed from the input features prior to training to avoid trivial leakage; the 13 raw features reported in Table~\ref{tab:data} reflect the original 14 predictive attributes minus \texttt{education} (the redundant categorical version of \texttt{education-num}), with \texttt{education-num} and \texttt{hours-per-week} subsequently moved to the target.

\textbf{German Credit.} This dataset, available from the UCI Machine Learning Repository, is used to predict creditworthiness, labeling individuals as having good or bad credit risk. It includes financial, personal, and loan-related information, providing a standard benchmark for classification and credit scoring problems. In our study, the prediction target is \texttt{CreditAmount}, an attribute representing the amount of credit (loan) requested by the individual. This attribute is removed from the input features prior to training so that no component of the target appears as a covariate.

\textbf{Communities and Crime.} Sourced from the UCI Machine Learning Repository, the Communities and Crime dataset combines socio-economic, law enforcement, and crime rate statistics for US communities. It is typically used to predict violent crime rates based on community-level indicators, making it valuable for regression analysis and policy research. The prediction target is \texttt{ViolentCrimesPerPop}, which represents the violent crime rate per capita in a community.

\subsection{MCE Computation}
The multicalibration error of a predictor $f$ is computed as a sample-mean estimator of $\|\Eh(f)\|_2$ evaluated over a random sample of $N_{\mathrm{tree}}$ regression trees with random node splits of depth $h$. In our experiments, $N_{\mathrm{tree}} = 100$ and $h = 3$. This random-tree class differs from the GBDT class used during training and is intentionally broader: evaluating MCE against this larger class stress-tests multicalibration beyond the training hypothesis space, providing a closer proxy to the population multicalibration error in Eq.~(1) than evaluating on the training class itself.

\subsection{Prediction Gap under Unit Weights}
Theorem~\ref{thm:main} guarantees that the gap $\|f_{t+1} - f_t\|_2$ decays at a rate of at least $O(1/\sqrt{T})$. Furthermore, Theorem~\ref{thm:linear} establishes that under smoothness conditions on the weak learners, this rate improves to linear (geometric decay).

To validate these bounds, Figure~\ref{fig:gap} shows the evolution of the prediction gap from two perspectives. The top row displays the gap on a linear scale, showing a clear overall decay toward zero across all datasets (with small per-round oscillations of $\le 10\%$ in some cases, consistent with Theorem~\ref{thm:main} bounding the \emph{minimum} and \emph{average} gap rather than the per-round gap), which empirically confirms the global convergence guarantees of Theorem~\ref{thm:main}.

To verify the stronger claim of linear convergence, the bottom row shows the same values on a semi-log scale. In this view, a linear trajectory implies geometric decay (i.e., $\|f_{t+1} - f_t\|_2 \propto \kappa^t$ for some $\kappa < 1$). We fit a linear regression to the tail of the log-gap sequence (dashed black lines) and observe near-linear behavior with $R^2 \in [0.78, 0.99]$. The lower $R^2$ on California Housing and Adult reflects a sublinear plateau, consistent with hard-split GBDTs only approximately satisfying the Lipschitz condition of Theorem~\ref{thm:linear}.

\subsection{Training MSE and MCE under the Rescaling Strategies}
\begin{figure}[!htbp]
\centering
\includegraphics[width=\linewidth]{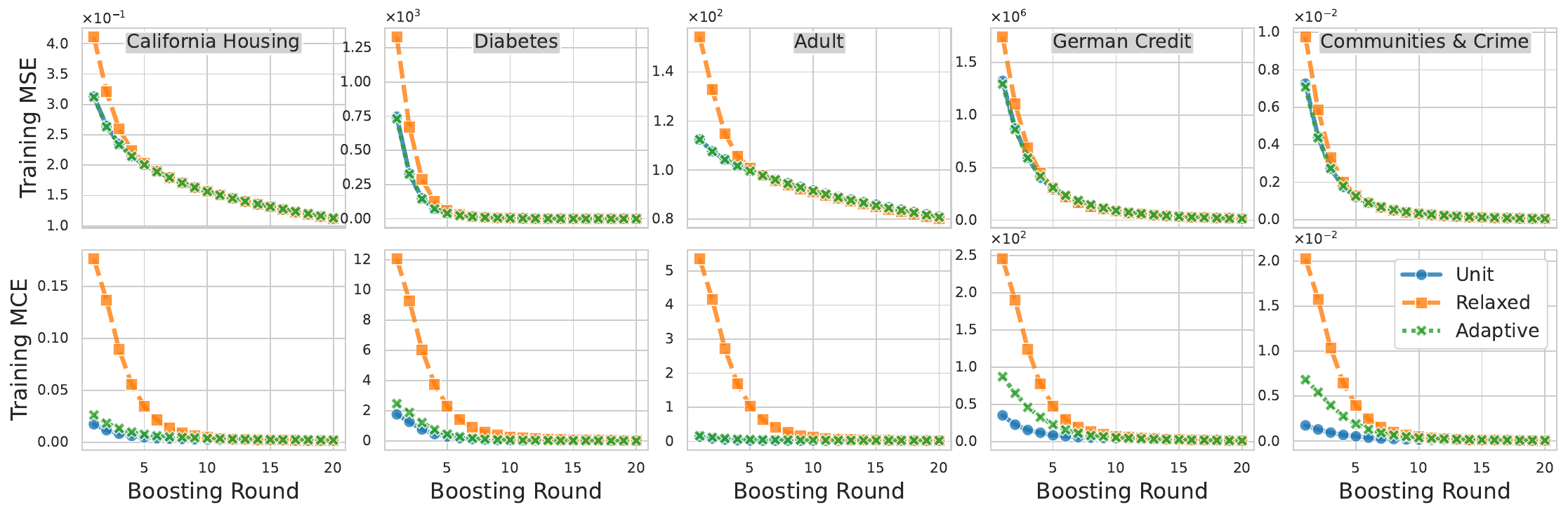}
\caption{Evolution of the training MSE (top row) and training MCE (bottom row) over 20 boosting rounds across five datasets.}
\label{fig:mse}
\end{figure}

Theorem~\ref{thm:adaptive} guarantees that the training Mean Squared Error (MSE) is non-increasing under adaptive weights (under relaxed weights, Theorem~\ref{thm:relaxed} gives only the weaker bound discussed in Section 4.2.2 and Appendix A.7), and Theorems~\ref{thm:relaxed} and~\ref{thm:adaptive} guarantee that the empirical multicalibration error converges to zero, with an explicit rate for adaptive weights (Theorem~\ref{thm:adaptive}). Thus, we investigate the reduction of the MSE and the Multicalibration Error (MCE) under the three considered rescaling strategies.

The top row of Figure~\ref{fig:mse} shows that every strategy yields a monotonic decrease in training MSE across all datasets. This confirms that the dynamic system remains stable regardless of the chosen rescaling. Interestingly, the \textbf{Relaxed} strategy (orange squares) shows a slightly slower initial convergence compared to \textbf{Unit} and \textbf{Adaptive} approaches; this is expected as the weights $w_t = 1 - (t + 2)^{-3}$ are smaller in the early iterations, effectively damping the update step. On the other hand, the \textbf{Adaptive} strategy (green crosses) matches or exceeds the convergence speed of the unit baseline, showing that optimizing the rescaling weights $w_t$ does not introduce instabilities.

While the MSE is a relevant loss, the main goal is to minimize the MCE. The bottom row of Figure~\ref{fig:mse} shows how the empirical MCE evolves over the rounds. Following our theoretical results, the plots show that the MCE has a consistent trend towards 0, often decreasing by some orders of magnitude (e.g., German Credit and Diabetes). The empirical results show a smooth and stable decay of multicalibration error to zero.

\subsection{Generalization}
We examine how the theoretical convergence guarantees translate to unseen data; Figure~\ref{fig:test} and its discussion are in Section 6. To make results more robust, we repeat the experiment for 20 different seeds and aggregate the test performance over the different runs. Each seed corresponds to a fresh random 80/20 train/test split (uniform random shuffle, no stratification, since all targets are continuous regression values). Figure~\ref{fig:test} plots the average test MCE across iterations, with star markers indicating the optimal stopping point (minimum average error) for each strategy. Shaded bands around each curve indicate $\pm 1$ standard deviation across the 20 seeds, so that the comparative claims in Section 6 can be assessed against random fluctuations.

Comparing the strategies shows a trade-off between efficiency and stability. The \textbf{Unit} and \textbf{Adaptive} strategies (blue and green lines) are highly efficient, typically reaching their optimal performance within the first 1-8 rounds (depending on the dataset). However, their aggressive updates can be harmful in some cases: this is most evident in the \emph{German Credit} dataset, where they both overfit almost immediately. In contrast, the \textbf{Relaxed} strategy (orange dashed line) acts as an implicit regularizer. By reducing the early updates, it decays more gradually, shifting the optimal stopping point to later rounds (e.g., around round 12 in \emph{California Housing} and round 5 in \emph{German Credit}; on \emph{Adult} it does not reach its minimum within $T = 20$). This makes the Relaxed strategy more robust to the choice of stopping time, reducing the risk of initial overfitting.

\section{Extended Related Work}\label{app:rel}
This appendix complements the discussion of related work in Section 2.

\textbf{Foundations of Multicalibration.} \citet{hebert2018} introduced multicalibration, showing that a predictor can be calibrated across a potentially exponential family of intersecting sub-populations using an iterative postprocessing algorithm (HKRR). \citet{dwork2021} later unified this line of work under outcome indistinguishability, placing multicalibration as a relaxation of indistinguishability from the true distribution. Related notions include multiaccuracy \citep{kim2019} and low-degree multicalibration \citep{gopalan2022b}, which weaken calibration constraints to low-degree polynomial moments. Subsequent work suggested that multicalibration does not need to arise solely from post-hoc correction, but can instead emerge as a natural fixed point of empirical risk minimization (ERM). \citet{gopalan2022a} introduced omniprediction, showing that multicalibrated predictors support optimal decision-making across a broad class of convex loss functions \citep{hu2023}. Building on this perspective, \citet{globus2023} established an equivalence between multicalibration and Bayes optimality under a swap-regret condition, leading to swap-agnostic learning \citep{gopalan2023}. \citet{globus2023} further connected multicalibration to boosting for squared-error regression, through repeated calls to a weak learner on subsets of the data distribution. Our work is different in providing the first finite-time convergence rate guarantees for recently proposed multi-round multicalibration gradient boosting algorithms \citep{tax2026} and considering both classification and regression tasks. Generalization guarantees for multicalibration have concerned discretized predictors \citep{hebert2018,globus2023} or uniform convergence over a given predictor class \citep{shabat2020}; the class of predictors generated by multi-round boosting, in which the prediction is fed back as a feature, is data-dependent and is handled in Section 5 by a covering argument across rounds.

\textbf{Multicalibration via Gradient Boosting.} Motivated by these connections, recent work has explored achieving multicalibration directly within gradient boosting. These approaches fit residuals using ensembles of decision trees, where each boosting round is defined with respect to the predictor from the previous round. A key feature is that trees are trained on an extended input space that augments the original features with the current prediction. \citet{jin2025} proposed such a method and showed that approximate multicalibration can be achieved in a single round under a restrictive loss saturation condition, which requires that further rounds yield only marginal improvement. Similarly, \citet{tax2026} introduced MCGrad, a recursive GBDT method deployed at web-scale. They showed that the empirical multicalibration error can be bounded by a function of the gap between successive predictions, but did not establish that this gap converges to zero.

Despite strong empirical performance, existing gradient boosting approaches lack convergence guarantees for arbitrarily many rounds. Current guarantees either rely on single-round analysis under a loss saturation condition \citep{jin2025} or assume diminishing gaps between successive predictions \citep{tax2026}. In contrast, we study the finite-time dynamics of multi-round boosting and derive convergence and rate guarantees.

\textbf{Extensions Beyond the Mean.} Although our analysis focuses on mean multicalibration, prior work has extended calibration guarantees to higher-order properties. The HappyMap framework \citep{deng2023} generalizes calibration targets to arbitrary statistics (e.g., quantiles), while \citet{jung2021} introduced moment multicalibration to control variance. \citet{gupta2022} and \citet{jung2023} further developed multivalidity for constructing prediction intervals, primarily in online learning settings. Our results complement these advances by providing the first rigorous convergence rates for the batch boosting dynamics underlying practical multicalibration methods.

\end{document}